%% file: 0-main.tex
\documentclass{article}
\PassOptionsToPackage{numbers,sort&compress}{natbib}
\usepackage[preprint]{neurips_2026}
\input{preamble}

\title{Conversational Generalized Linear Bandits with Hybrid Feedback}

\author{%
{\bfseries
Yurui Zou$^{1}$ \quad
Qirun Zeng$^{2}$ \quad
Jinhang Zuo$^{2}$}\\[0.7em]
{\small
$^{1}$University of Science and Technology of China, Hefei, China} \\[-0.05em]
{\small
$^{2}$City University of Hong Kong, Hong Kong SAR, China} \\[0.4em]
{
\small
\texttt{zyurui@mail.ustc.edu.cn}
} \\[-0.05em]
{
\small
\texttt{qirun.zeng@my.cityu.edu.hk}
} \\[-0.05em]
{\small
\texttt{jinhang.zuo@cityu.edu.hk}}
}

\begin{document}
\maketitle

\begin{abstract}
Conversational bandits improve recommendations by actively querying users about key-terms to elicit their item preferences. However, existing methods typically assume homogeneous feedback for items and key-terms, using either absolute rewards or pairwise comparisons. In practice, item feedback often takes the form of clicks or ratings, while comparisons are more natural for key-terms. To address this gap, we study conversational generalized linear bandits with hybrid feedback, combining absolute item rewards and relative key-term comparisons. We propose Conversational UCB with Hybrid Feedback (HyConUCB), which constructs unified confidence bounds from heterogeneous feedback and adaptively selects key-term comparisons by balancing information gain and directional coverage. Through a refined regret analysis that jointly exploits both feedback sources, we establish a regret bound that improves upon reward-only and adapted dueling-only benchmarks. Experiments on synthetic and real-world datasets validate the effectiveness of our algorithm.

\end{abstract}

\input{1-intro}

\input{2-problem}

\input{3-confidence-set}

\input{4-algorithm}

\input{5-regret-analysis}

\input{6-reverse-feedback}
\input{7-syn-experiments}

\input{8-real-experiments}

\input{9-conclusion}

\label{page:main-end}

\clearpage
\bibliographystyle{unsrtnat}
\bibliography{refs}

\clearpage
\appendix
\input{appendix/1-confidence-set}
\input{appendix/2-regret-analysis-theorem}
\input{appendix/3-regret-analysis-lemma}
\input{appendix/4-additional-discussions}

\end{document}

%% file: preamble.tex
\usepackage[utf8]{inputenc}
\usepackage[T1]{fontenc}
\usepackage{amsmath,amssymb,mathtools,amsthm}
\usepackage{graphicx,xcolor}
\usepackage{algorithm,algorithmic}
\usepackage{float}
\usepackage{textcomp,bm}
\usepackage{booktabs,subcaption}
\usepackage{microtype}
\usepackage{url}
\usepackage[hidelinks]{hyperref}

\usepackage{cleveref}
\theoremstyle{plain}
\newtheorem{theorem}{Theorem}[section]
\newtheorem{lemma}[theorem]{Lemma}

\AddToHook{env/theorem/begin}{\crefalias{section}{theorem}\crefalias{theorem}{theorem}}
\AddToHook{env/lemma/begin}{\crefalias{section}{lemma}\crefalias{theorem}{lemma}}
\AddToHook{env/assumption/begin}{\crefalias{assumption}{assumption}}

\DeclareMathOperator{\tr}{tr}

\DeclareMathOperator*{\argmax}{arg\,max}
\DeclareMathOperator*{\argmin}{arg\,min}

\newcommand{\ts}{\bm{\theta}^\star}

\hypersetup{
  pdftitle={Conversational Generalized Linear Bandits with Hybrid Feedback},
  pdfauthor={Qirun Zeng, Eric He, Richard Hoffmann, Xuchuang Wang, Jinhang Zuo},
  pdfsubject={Conversational generalized linear bandits with hybrid feedback},
  pdfkeywords={Conversational Bandits, Hybrid Feedback}
}

%% file: 1-intro.tex
\section{Introduction}
\label{sec:intro}
Conversational recommender systems (CRSs) improve recommendation quality by actively interacting with users to elicit their preferences. Unlike traditional recommender systems that primarily learn from feedback on recommended items, CRSs can proactively query users about key-terms associated with groups of items, such as movie genres or news categories. These additional interactions help alleviate the cold-start problem by enabling more efficient preference learning. To formalize this process, \citet{zhang2020conversational} introduced conversational contextual bandits, where the learner sequentially recommends items (i.e., arms) and observes user feedback in the form of rewards. In addition, the learner can occasionally initiate conversations by querying users about selected key-terms and collecting additional feedback. By jointly leveraging item-level and key-term-level feedback, conversational bandits can learn user preferences more efficiently and improve recommendation performance.

Existing conversational bandit methods primarily consider two types of user feedback: absolute rewards and pairwise comparisons. Early studies on conversational contextual bandits \citep{zhang2020conversational} assume that users provide absolute feedback, such as clicks or ratings, for both recommended items and queried key-terms. More recently, \citet{yang2024conversational} introduced conversational dueling bandits, where user preferences are elicited through pairwise comparisons of items and key-terms. While both feedback types provide valuable preference information, assuming a homogeneous feedback form may not fully reflect practical recommendation scenarios. In particular, users can readily provide absolute feedback on recommended items, whereas pairwise comparisons are often more natural for abstract key-terms. This motivates the study of conversational bandits with hybrid feedback, combining absolute item rewards and relative key-term comparisons.

For example, in a movie recommender system illustrated in Fig.~\ref{fig:illustrative_figure}, users may click on or rate recommended movies, providing absolute feedback on items. Meanwhile, the system can occasionally ask users to compare movie genres, such as “Do you prefer comedies or horror movies?”, obtaining relative feedback on key-terms. Combining these two feedback types allows the system to learn user preferences more effectively without requiring pairwise comparisons for every recommendation. However, integrating heterogeneous feedback into conversational bandits with theoretical guarantees remains challenging, as absolute rewards and pairwise comparisons follow different statistical observation models. Existing work has only partially addressed this problem \citep{xie2021comparison}.
\begin{figure}[H]
  \centering
  \includegraphics[width=0.85\linewidth]{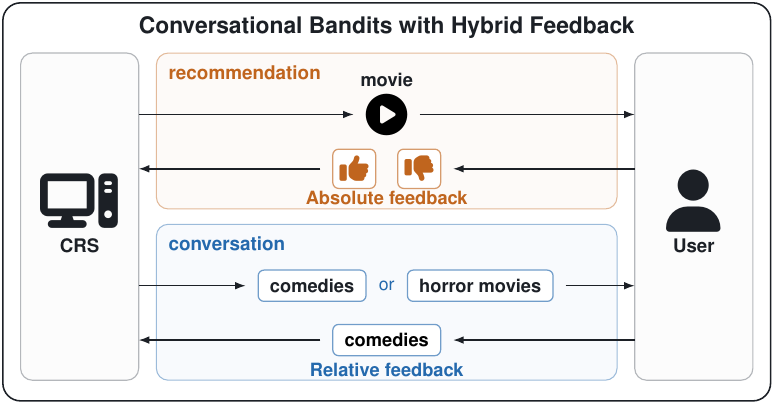}
  \caption{An illustration of conversational bandits with hybrid feedback in a movie recommender system. Users provide absolute feedback on recommended movies through likes or dislikes. The system also occasionally asks users to compare movie genres, such as comedies and horror movies, to elicit relative preferences.}
  \label{fig:illustrative_figure}
\end{figure}

\begin{table}[t]
\centering
\caption{Regret bound comparison between HyConUCB and
homogeneous-feedback benchmarks.}
\label{tab:regret_comparison}
\small
\renewcommand{\arraystretch}{1.15}
\begin{tabular*}{\textwidth}{@{\extracolsep{\fill}}lcc@{}}
\toprule
Algorithm & Leading term & Lower-order term \\
\midrule
Reward-only~\citep{lee2024a}
& $\tilde O\left(d\sqrt{T/\kappa_\star}\right)$
& $\tilde O(\kappa d^2)$ \\
ConDuel (adapted)~\citep{yang2024conversational}
& $\tilde O\left(\kappa d\sqrt{T/\kappa_\star}\right)$
& $\tilde O(\kappa^2d^2)$ \\
HyConUCB (ours)
& $\tilde O\left(
    d\sqrt{\frac{T/\kappa_\star}{1+\eta \bar v^{-1}}}
\right)$
& $\tilde O\left({d^2}/(\kappa^{-1}+\eta)\right)$ \\
\bottomrule
\end{tabular*}
\par\smallskip
\begin{minipage}{\textwidth}
\footnotesize
We consider $T$ rounds with $b(t)=\lfloor\rho t\rfloor$ cumulative conversational queries, where $\rho$ controls the query frequency. We report the case $r=d$, where $d$ is the feature dimension and $r$ is the dimension spanned by conversational queries (see Eq.~\eqref{eq:rate} for general $r$). Here, $\kappa^{-1}$ and $\kappa_\star^{-1}$ denote the worst-case and average optimal-arm curvatures, respectively; $\bar v$ is the maximum reward curvature; and $\eta=\sigma\rho d/4$ measures conversational information. The ConDuel bound is obtained by adapting its analysis to our setting. $\tilde O$ hides logarithmic factors and fixed constants.

\end{minipage}
\end{table}

Motivated by these observations, we study conversational generalized linear bandits with hybrid feedback, focusing on absolute arm rewards and relative key-term comparisons~\footnote{Our framework also accommodates the reverse configuration,
with relative arm feedback and absolute key-term feedback;
the corresponding results are presented in
Section~\ref{sec:reverse_feedback}.}. We adopt generalized linear models (GLMs) to capture diverse user responses within a unified framework. To overcome the limitations of existing conversational dueling bandits \citep{yang2024conversational}, whose confidence bounds depend on the global worst-case curvature, we extend the confidence sequence construction of \citet{lee2024a} to conversational bandits setting, handling heterogeneous arm and key-term feedback without this dependence in the confidence radius. Building on these confidence sequences, we propose Conversational UCB with Hybrid Feedback (HyConUCB), which selects arms using upper confidence bounds and adaptively queries key-term pairs by balancing information gain and directional coverage. Through a refined regret analysis that jointly exploits information from both feedback sources, we establish improved regret guarantees over reward-only and adapted dueling-only benchmarks, as summarized in Table~\ref{tab:regret_comparison}. Our analysis explicitly quantifies the benefits of conversational feedback. We also demonstrate the effectiveness of HyConUCB through experiments on synthetic and real-world datasets.

Our main contributions are summarized as follows:
\begin{itemize}
    \item We formulate \textbf{conversational generalized linear bandits (GLBs) with hybrid feedback}, integrating absolute arm rewards and relative key-term comparisons under a unified generalized linear model. This framework captures heterogeneous user feedback commonly encountered in conversational recommendation systems.

    \item We develop \textbf{HyConUCB}, which constructs time-uniform confidence sequences from heterogeneous feedback without relying on global worst-case curvature scaling. The algorithm combines UCB-based arm selection with an adaptive key-term querying strategy that balances information gain and directional coverage.

    \item We establish \textbf{improved regret guarantees} through a refined analysis that jointly exploits information from both feedback sources. In particular, we develop a new elliptical potential bound that incorporates conversational information despite the nonlinear curvature of GLMs, yielding improved regret bounds over reward-only and adapted dueling-only benchmarks.

    \item We conduct extensive experiments on \textbf{synthetic data and two real-world datasets}, Yelp and xRouteBench, covering restaurant recommendation and LLM routing. The results demonstrate the effectiveness of HyConUCB in reducing cumulative regret and improving preference estimation, highlighting the benefits of hybrid feedback across different settings.
\end{itemize}

\section{Related Work} 
Our work builds on several research areas, and we review some recent work in the most related areas.

\noindent\textbf{Conversational Bandits.} 
\citet{zhang2020conversational} introduced ConUCB, which augments item feedback with key-term queries to accelerate preference learning. 
Building on this framework, \citet{xie2021comparison} proposed RelativeConUCB to incorporate relative key-term feedback. Although its feedback configuration is similar to ours, RelativeConUCB is a simple empirical extension of ConUCB~\citep{zhang2020conversational} and does not establish regret guarantees.
\citet{wang2023efficient} proposed ConLinUCB, jointly estimating user preferences from arm-level and key-term-level feedback. 
Based on this framework, \citet{yang2024conversational} studied conversational dueling GLBs and \citet{liu2025leveraging} developed more effective query strategies for linear conversational bandits.
However, the former's regret analysis yields a suboptimal bound in the GLM setting and the latter's techniques cannot be directly extended to GLBs because of the nonlinear terms introduced by GLMs. We therefore develop new techniques for regret analysis in conversational GLBs. To the best of our knowledge, we provide the first regret guarantees for conversational bandits with hybrid feedback.

\noindent\textbf{Generalized Linear Bandits.} In a canonical generalized linear model, an exponential-family response has natural parameter $\bm{x}^\top\ts$ and mean $\mu(\bm{x}^\top\ts)$, where $\ts$ is unknown~\citep{mccullagh2019generalized}. \citet{filippi2010parametric} introduced GLM-UCB for generalized linear bandits and established regret guarantees. Subsequent work has improved confidence estimation and regret analysis in GLBs, including through the use of self-concordance~\citep{russac2021self}. In particular, \citet{lee2024a} developed a unified likelihood-ratio confidence sequence for GLMs, with convex confidence regions and an ellipsoidal relaxation under self-concordance. We adapt this confidence construction to the joint likelihood of absolute arm rewards and relative key-term feedback.

\noindent\textbf{Relative and Hybrid Feedback.} Relative feedback has been widely studied in bandit learning and preference-based learning~\citep{yue2011beat,yue2012k,sui2018advancements}. 
Relative key-term feedback has also been studied in conversational recommendation~\citep{xie2021comparison,yang2024conversational}. 
More recently, \citet{wang2025fusing} studied stochastic bandits that combine reward and dueling feedback. \citet{he2026learning} considered hybrid bandits with heterogeneous offline and online feedback, and 
\citet{zeng2026best} took a step further by studying best-arm identification in GLBs with hybrid feedback, where a shared latent parameter $\ts$ links absolute rewards and relative comparisons. 
We adopt this shared-parameter formulation to represent user preferences, while focusing on regret minimization and studying hybrid feedback across the item and key-term levels: absolute rewards on items and relative comparisons between key-terms.

%% file: 2-problem.tex
\section{Problem Formulation}
\label{sec:problem}
\subsection{Generalized Linear Model}
We formulate both types of feedback using GLMs, so we begin by introducing the generalized linear model (GLM). In a canonical GLM, the conditional distribution of observation $y$, given a natural parameter $u\in\mathbb{R}$, belongs to an exponential family and is specified by
\begin{equation*}
 p(y\mid u)=\exp\left\{\frac{yu-m(u)}{a}+c(y)\right\}.
\end{equation*}
Here $m$ is the log-partition function, $a$ is the known dispersion factor and $c$ is the log base measure. 
We use $p(y\mid u)$ to denote the conditional probability density or mass function of $y$ given $u$.
A fundamental property of GLMs is that the conditional mean of $y$ is given by
\begin{equation*}
    \mathbb E[y\mid u]=m'(u)\eqqcolon \mu(u),
\end{equation*}
where $\mu(\cdot)$ is the mean function (inverse link function). For convenience, we define $\nu(u) \coloneqq \frac{\dot\mu(u)}{a}$.

\subsection{Conversational Generalized Linear Bandits with Hybrid Feedback}
We consider a sequential decision-making problem where the learner interacts with a user over $T$ rounds.
Let $\mathcal{A}$ denote a set of $K$ arms and $\mathcal{K}$ denote a set of $P$ key-terms. The relationship between the arms and the key-terms can be characterized by a weighted bipartite graph $(\mathcal{A}, \mathcal{K}, \bm{W})$, whose nodes are divided into two sets $\mathcal{A}$ and $\mathcal{K}$, and
weighted edges are represented by the matrix $\bm{W}\eqqcolon[W_{a,k}]_{K\times P}$ with $W_{a,k}$ representing the relationship between arm $a$ and key-term $k$. 

At round $t$, the learner observes a candidate set $\mathcal{A}_t\subseteq\mathcal{A}$. It then selects an arm $a_t\in\mathcal{A}_t$ and receives an absolute reward $r_t$ from the user. Each arm is associated with a $d$-dimensional feature vector $\bm{x}_{a}$. The corresponding feature vector pool at round $t$ is $\mathcal{X}_t$. We write $\bm{x}_t=\bm{x}_{a_t}$ for the feature vector of the selected arm. Let $\|\bm{x}_{a}\|_2\leq 1$ for every arm feature. The user's preference is represented by a vector $\ts\in\Theta$, where $\Theta\coloneqq\{\bm{\theta}\in\mathbb R^d:\|\bm{\theta}\|_2\le S\}$. Then the reward is defined as 
\begin{equation}\label{eq:reward}
r_t \sim p_{\mathrm{c}}(\cdot | \bm{x}_t^\top \ts).    
\end{equation}

The learner obtains additional user feedback through key-term conversations. We set the conversational feedback to be relative, that is, the learner selects a pair of key-terms and receives the dueling result of them.

We define the key-term feature vector as
\[
\tilde {\bm{x}}_{k}
\coloneqq \frac{\sum_{a\in\mathcal A}W_{a,k}\bm{x}_{a}}
{\sum_{a\in\mathcal A}W_{a,k}}.
\]
A nondecreasing integer-valued function $b(t)$, with $b(0)=0$, is defined to model the frequency of conversations. Before selecting an arm, the learner conducts $q_t=b(t)-b(t-1)$ conversations. For each $i\in\{1,2,...,q_t\}$, the learner selects a key-term pair $(k_{t,i},k'_{t,i})\in\mathcal{K}^2$. Define its comparison feature vector as
\begin{equation*}
    \tilde {\bm{d}}_{t,i} \coloneqq \tilde {\bm{x}}_{k_{t,i}}-\tilde {\bm{x}}_{k'_{t,i}}.
\end{equation*}
Thus $\|\tilde {\bm{d}}_{t,i}\|_2\le2$, and the learner observes the dueling feedback $\tilde o_{t,i}\in\{0,1\}$ with distribution
\begin{equation}\label{eq:dueling_feedback}
    \tilde o_{t,i} \sim p_{\rm d} \left(\cdot\mid\tilde {\bm{d}}_{t,i}^\top\ts\right),
\end{equation}
where $\tilde o_{t,i}=1$ indicates preference for $k_{t,i}$. 

Both arm reward and conversational dueling feedback are modeled as GLMs with a shared parameter $\ts$, but possibly different mean functions $\mu_{\rm c}$ and $\mu_{\rm d}$. We use subscripts $\rm c$ and $\rm d$ to distinguish reward and dueling quantities. We define the feedback modality to be $\rm c$ for absolute feedback and $\rm d$ for relative feedback. 

Let the query pool $\mathcal Q \coloneqq \{\bm{q}_e\}_{e=1}^{N_q}$ contain all pairwise differences $\tilde{\bm{x}}_k-\tilde{\bm{x}}_{k'}$ from the key-term pool, and let $U \coloneqq \operatorname{span} (\mathcal Q)$ and $r = \dim(U)$. Unlike prior works that assume full-dimensional conversational coverage ($r=d$)~\citep{yang2024conversational,wang2023efficient}, our algorithm also applies when conversations cover only a subspace ($r<d$).

Under the conversational GLBs with hybrid feedback problem described above, our objective is to minimize the cumulative regret in $T$ rounds. Let $\bm{x}_{t,\star}=\arg\max_{\bm{x}\in\mathcal X_t}\bm{x}^\top \ts$. Following the definition in \citet{filippi2010parametric}, we define the regret as
\[
R(T)=\sum_{t=1}^T\left\{
\mu_{\mathrm{c}}((\bm{x}_{t,\star})^\top \ts)-\mu_{\mathrm{c}}(\bm{x}_t^\top \ts)\right\}.
\]

Following \citet{lee2024a,zeng2026best}, we assume that the log-partition functions $m_{\rm c},m_{\rm d}$ are three times continuously differentiable, with $\dot \mu_{\rm c},\dot\mu_{\rm d}>0$. For modality $j\in\{\mathrm{c},\mathrm{d}\}$, we use a self-concordance bound $R_s$ and curvature bound $R_{\dot \mu}$:
\begin{equation*}
\begin{aligned}
 |\ddot\mu_j(u)|&\le R_s\dot\mu_j(u),\qquad R_{\dot\mu}=\max_{j\in\{\mathrm{c},\mathrm{d}\}}\sup_u\dot\mu_j(u).
\end{aligned}
\end{equation*}

%% file: 3-confidence-set.tex
\section{Hybrid Conversational UCB Algorithm}
\label{sec:algorithm}

\subsection{Confidence Sequences}
\label{sec:confidence}

We construct confidence sequences for the unknown user preference vector $\ts$. Following the procedure in \citet{zeng2026best} and \citet{lee2024a}, we first give a likelihood-ratio–based confidence sequence. Then we relax it to ellipsoidal confidence sequence which is more suitable for our algorithm design and analysis.

Define the cumulative negative log-likelihood function at round $t$ as
\begin{equation}\label{eq:joint_loss}
\begin{aligned}
\mathcal L_t(\bm{\theta})
={}&\sum_{s<t}\frac{-r_s\bm{x}_s^\top\bm{\theta}+m_{\mathrm{c}}(\bm{x}_s^\top\bm{\theta})}{a_{\mathrm{c}}} \\
&+ \sum_{s\le t}\sum_{i=1}^{q_s} 
\frac{-\tilde o_{s,i}\tilde {\bm{d}}_{s,i}^\top\bm{\theta}
+m_{\mathrm{d}}(\tilde {\bm{d}}_{s,i}^\top\bm{\theta})}{a_{\mathrm{d}}}.
\end{aligned}
\end{equation}
Following \citet{lee2024a}, we define $\hat {\bm{\theta}}_t \in \arg\min_{\bm{\theta} \in \Theta} \mathcal L_t(\bm{\theta})$ and
\begin{equation*}
\mathcal C_t \coloneqq \{\bm{\theta}\in\Theta:
\mathcal L_t(\bm{\theta})-\mathcal L_t(\hat{\bm{\theta}}_t)\le\beta_t^2\},
\end{equation*}
where 
\[
\begin{aligned}
\beta_t^2 \coloneqq{}&\log \frac2\delta + \inf_{c \in(0,1]}\left\{d \log \frac1c + 2 S L_t c \right\}\\
\le{}& \log\frac2\delta+ d \log \left(e \vee \frac{2 e S L_t}{d} \right).
\end{aligned}
\]
and $L_t$ is the Lipschitz constant of $\mathcal L_t$ on $\Theta$.

\begin{lemma}\label{lem:confidence}
With probability at least $1-\frac\delta2$, we have $\ts \in \mathcal C_t$ for all $t\le T$.
\end{lemma}

\begin{proof}[Proof Sketch]
To prove it, we index the arm rewards and key-term comparisons in their actual order of observation. Before arm selection in round $t$, there are $N_t=t-1+b(t)$ observations. Write the resulting sequence as $(j_n,\bm{z}_n,y_n)_{n=1}^{N_t}$, where $j_n\in\{\mathrm{c},\mathrm{d}\}$ denotes the modality and $(\bm{z}_n,y_n)$ denotes $(\bm{x}_s,r_s)$ for an arm reward or $(\tilde{\bm{d}}_{s,i},\tilde o_{s,i})$ for a key-term comparison. Eventually, we have
\[
\mathcal L_t(\bm{\theta})
=\sum_{n=1}^{N_t}
\frac{-y_n\bm{z}_n^\top\bm{\theta}+m_{j_n}(\bm{z}_n^\top\bm{\theta})}{a_{j_n}}.
\]
This function has the same form as the cumulative negative log-likelihood function in \citet{zeng2026best}. Applying \citet[Lemma~3.1]{zeng2026best} gives \Cref{lem:confidence}.
\end{proof}

Let $\lambda>0$. The regularized joint Hessian at
$\hat{\bm{\theta}}_t$ is
\begin{equation}\label{eq:joint_hessian}
\begin{aligned}
\bm{H}_t&=\lambda \bm{I}_d+\nabla^2\mathcal L_t(\hat{\bm{\theta}}_t)\\
&=\lambda \bm{I}_d+
\sum_{s<t}\nu_{\mathrm{c}}(\bm{x}_s^\top\hat{\bm{\theta}}_t)\bm{x}_s\bm{x}_s^\top+\sum_{s\le t}\sum_{i=1}^{q_s}
\nu_{\mathrm{d}}(\tilde {\bm{d}}_{s,i}^\top\hat{\bm{\theta}}_t)
\tilde {\bm{d}}_{s,i}\tilde {\bm{d}}_{s,i}^\top.
\end{aligned}
\end{equation}
Following the ellipsoidal construction of
\citet[Theorem~3.2]{lee2024a}, let $\gamma_t^2 \coloneqq 2(1+2SR_s)(4\lambda S^2+\beta_t^2)$, we define
\begin{equation*}
\mathcal E_t \coloneqq \{\bm{\theta}\in\mathbb R^d:
\|\bm{\theta}-\hat {\bm{\theta}}_t\|_{\bm{H}_t}^2 \le \gamma_t^2\}.
\end{equation*}

\begin{lemma}\label{lem:ellipsoid}
For every $t \le T$, $\mathcal C_t \subseteq \mathcal E_t$.
Consequently, with probability at least $1-\frac\delta2$, $\ts \in \mathcal E_t$, for all $t\le T$.
\end{lemma}
The proof is given in Appendix~\ref{app:confidence_set}.

%% file: 4-algorithm.tex
\subsection{Algorithm Design}

\begin{algorithm}[t]
\caption{HyConUCB}
\label{alg:hybrid_glm}
\begin{algorithmic}[1]
\REQUIRE The weighted bipartite graph $(\mathcal{A},\mathcal{K},\bm{W})$, conversational frequency function $b(t)$, parameter norm bound $S$
\STATE Construct the query pool $\mathcal Q$.
\FOR{$t=1,\ldots,T$}
\STATE Set $q_t\gets b(t)-b(t-1)$.
\FOR{$i=1,\ldots,q_t$}
\STATE Solve query distribution $\bm{p}_{t,i}$ according to Eq.~\eqref{eq:query_selection}.
\STATE Sample $(k_{t,i},k'_{t,i})$ according to $\bm{p}_{t,i}$.
\STATE Query $(k_{t,i},k'_{t,i})$ and observe $\tilde o_{t,i}$, and set $\mathcal F_{t-1} \gets \mathcal F_{t-1} \cup (k_{t,i}, k'_{t,i}, \tilde o_{t,i})$.
\ENDFOR
\STATE Compute $\hat {\bm{\theta}}_t \in \arg\min_{\bm{\theta} \in \Theta} \mathcal L_t(\bm{\theta})$.
\STATE Select $a_t \in \arg\max_{a \in \mathcal{A}_t} \left\{\bm{x}_{a}^\top \hat {\bm{\theta}}_t + \gamma_t\| \bm{x}_{a} \|_{\bm{H}_t^{-1}}\right\}$.
\STATE Observe $r_t$ and set $\mathcal F_t \gets \mathcal F_{t-1} \cup (a_t, r_t)$.
\ENDFOR
\end{algorithmic}
\end{algorithm}
Building on the confidence sequences above, we present the Conversational UCB with Hybrid Feedback algorithm (HyConUCB; \Cref{alg:hybrid_glm}).
The central idea of HyConUCB is to balance information gain and directional coverage in key-term selection, while employing the UCB principle for arm selection.

At each round $t$, the learner first conducts $q_t=b(t)-b(t-1)$ conversations to gather additional information about user preferences through key-term comparisons. It then updates its parameter estimate using both arm rewards and key-term comparisons, and selects an arm according to the UCB principle.

We first describe the \emph{key-term} selection rule. For information gain, we adopt the MaxInp principle~\citep{saha2021optimal}. This principle selects the key-term that maximizes the information uncertainty, i.e., \((k_{t,i},k'_{t,i})\in\arg\max_{k,k'\in\mathcal K}
\|\widetilde{\bm{x}}_k-\widetilde{\bm{x}}_{k'}\|_{\bm{M}_{t,i}^{-1}},\)
where $\bm{M}_{t,i}$ is the joint Hessian immediately before query $i$ in round $t$.

For each query $\bm{q}_e\in\mathcal{Q}$, calculate
\[
a_{t,i,e}=\|\bm{q}_e\|^2_{\bm{M}_{t,i}^{-1}}
\quad \text{and}\quad
s_{t,i,e}=\frac{a_{t,i,e}}{\max_f a_{t,i,f}}\in[0,1],
\]
where $a_{t,i,e}$ is the MaxInp score of $\bm{q}_e$.

Next we consider directional coverage. For a query distribution $\bm{p}$ on $\mathcal{Q}$, the matrix $\bm{V}(\bm{p})=\sum_e p_e\bm{q}_e\bm{q}_e^\top$ summarizes how the sampled query features cover different directions. Its eigenvalues reflect how strongly the queries explore the corresponding eigenvector directions. To encourage coverage of every direction in $U$, we use the minimum eigenvalue of this matrix restricted to $U$. Define
\[
j(\bm{p})=\lambda_{\min,U}\!\left(\bm{V}(\bm{p})\right),
\]
where $\lambda_{\min,U}$ denotes the minimum eigenvalue restricted to $U$.

Finally, we solve
\begin{equation}\label{eq:query_selection}
\bm{p}_{t,i}\in\argmax_{\bm{p}\in\Delta_{N_q}}
\left\{
\sum_e p_es_{t,i,e}
+
\xi\log j(\bm{p})
\right\},
\qquad 0<\xi<1.
\end{equation}
According to $\bm{p}_{t,i}$, we sample $\tilde {\bm{d}}_{t,i}$ from $\mathcal Q$.

The first term in Eq.~\eqref{eq:query_selection} favors queries with large MaxInp scores, while the second term encourages information coverage in the least-observed direction. The parameter $\xi$ controls this trade-off.
So the resulting query rule balances the information gain and directional coverage. It also guarantees the conditional information lower bound in Lemma~\ref{lem:information}. The proof is provided in Appendix~\ref{app:algorithm_proofs}.

After completing the conversational queries, the learner computes the updated parameter estimate $\hat{\bm{\theta}}_t$ using all feedback observed so far. Based on \Cref{lem:ellipsoid}, we define the UCB score of each arm $a\in\mathcal A_t$ as
\[
u_{t,a}\coloneqq
\bm{x}_{a}^{\top}\hat{\bm{\theta}}_t
+\gamma_t\|\bm{x}_{a}\|_{\bm{H}_t^{-1}}.
\]
The learner then selects an \emph{arm} that maximizes the UCB score, i.e.,
\[
a_t\in\argmax_{a\in\mathcal A_t}u_{t,a},
\]
breaking ties arbitrarily. Finally, the learner observes the arm reward $r_t$, which is incorporated into the parameter estimation in subsequent rounds. 

\begin{lemma}\label{lem:information}
Let
\(\mathcal F_t\coloneqq\sigma\Bigl(a_s,r_s),
\{(k_{s,i},k'_{s,i},\widetilde o_{s,i})\}_{i=1}^{q_s}
:\,1\le s\le t
\Bigr)\) denote the $\sigma$-algebra generated by the history through the end of round $t$. Under the query rule in~Eq.~\eqref{eq:query_selection}, there exists
a constant $\sigma>0$ such that, almost surely,
for every $t\in[T]$,
\begin{equation}\label{eq:fisher_condition}
\begin{gathered}
\mathbb E\!\left[
\sum_{i=1}^{q_t}
\nu_{\mathrm{d}}(\tilde {\bm{d}}_{t,i}^{\top}\ts)
\tilde {\bm{d}}_{t,i}\tilde {\bm{d}}_{t,i}^{\top}
\,\middle|\,\mathcal F_{t-1}
\right]
\succeq
\sigma q_t\bm{P}_U,
\end{gathered}
\end{equation}
where $\bm{P}_U$ is the orthogonal projection matrix onto $U$.
\end{lemma}

%% file: 5-regret-analysis.tex
\section{Regret Analysis}

In this section, we establish a regret bound for HyConUCB. We use the notations
\begin{equation*}
\begin{aligned}
\kappa&=\max_{t\le T,\,\bm{x}\in\mathcal X_t,\,\bm{\theta}\in\Theta}
\frac1{\dot\mu_{\rm c}(\bm{x}^\top\bm{\theta})},\qquad
\kappa_\star=T/\sum_{t=1}^T
\dot\mu_{\rm c}((\bm{x}_{t,\star})^\top\ts),\\
\bar v&=\max_{t\le T,\,\bm{x}\in\mathcal X_t}\dot\mu_{\rm c}(\bm{x}^\top\ts),\qquad
R_{\mu,\star}=\sup_{t\le T,\,\bm{x}\in\mathcal X_t}|\mu_{\rm c}(\bm{x}^\top\ts)|.
\end{aligned}
\end{equation*}
For the analysis, set $\bar\gamma_T=\max_{t\le T}\gamma_t$, $q_{\max}=\max\{1,q_1,\ldots,q_T\}$ and
$G_T=\max\left\{R_s\bar\gamma_T,\sqrt{R_{\dot\mu}/a_{\rm c}}\right\}$.

For $\varepsilon>0$, define
\begin{equation*}
\begin{aligned}
    n_0&\triangleq\frac{4}{\sigma} \max \left \{8q_{\max} \frac{R_{\dot \mu}}{a_{\rm d}} \log\frac{2rT}{\delta},
16\left(1+\max_{s\le T}\beta_s^2\right)
\left(\frac1\varepsilon+\frac{R_s}{2}\right)^2\right\},\\
 t_0&\triangleq\big|\{t\le T:b(t)<n_0\}\big|.
\end{aligned}
\end{equation*}

Let
\begin{equation}\label{eq:potential_bound}
\begin{aligned}
D_T(h)&\coloneqq (d-r) \log\left(1+\frac{Th}{2\lambda(d-r)}\right)\\
&+r\sum_{t=1}^T\log\left(1+
\frac{h}
{2\lambda r+(t-1)h+e^{-R_s\varepsilon}\sigma rb(t)/2}
\right).
\end{aligned}
\end{equation}
We now establish the theoretical guarantees of \Cref{alg:hybrid_glm}.

\begin{theorem}\label{thm:regret}
With probability at least $1-\delta$, \Cref{alg:hybrid_glm} satisfies
\begin{equation}\label{eq:main_bound}
\begin{aligned}
R(T)&\le{}2R_{\mu,\star}t_0
+8e^{3/2}\sqrt{a_{\rm c}}\bar\gamma_T\sqrt{D_T(\bar v/a_{\rm c})T/\kappa_\star}\\
&+16e^{3}a_{\rm c}R_s\bar\gamma_T^2D_T(\bar v/a_{\rm c})+\frac{32R_{\mu,\star}a_{\rm c}\kappa G^2_T}{\log2}
D_T(1/(a_{\rm c}\kappa)).
\end{aligned}
\end{equation}
\end{theorem}

\begin{proof}[Proof Sketch]
There are two main challenges in the regret analysis of conversational generalized linear bandits.
First, the information contributed by each observation is weighted by a nonlinear term, that is, the joint Hessian contains curvature weights such as $\nu_{\rm c}(\bm{x}_s^\top\hat{\bm{\theta}}_t)$. The analysis of ConDuel~\citep{yang2024conversational} handles this dependence by scaling based on the global worst-case curvature, which introduces a factor $\kappa$ into the confidence width and leads to a suboptimal confidence set.

Thanks to the confidence sequences in \citet{lee2024a}, we can now obtain tighter confidence sequences whose radii do not explicitly depend on $\kappa$.
The second challenge is to retain the information from both arm rewards and key-term comparisons when bounding the elliptical potential sum
\begin{equation}\label{eq:challenge}
    \sum_{t=1}^T\|\bm{x}_t\|^2_{\bm{H}_t^{-1}}.
\end{equation}
Although ConDuel and ConLinUCB combine both sources when estimating the user parameter, key steps in their regret analyses lower-bound the joint information matrix using only one source~\citep{yang2024conversational,wang2023efficient}. Such bounds do not retain the full information from both feedback sources. Moreover, the techniques developed for linear conversational bandits, including those of \citet{liu2025leveraging}, do not directly handle the curvature weights in our joint Hessian. These limitations motivate us to develop more effective techniques for regret analysis.

Our first step is to construct a tractable lower surrogate for $\bm{H}_t$. Let $\bar{\bm{\theta}}_s$
minimize $\dot\mu_{\rm c}(\bm{x}_s^\top\bm{\theta})$ over
$\bigcup_{\ell=s}^T\mathcal C_\ell$, as in~Eq.~\eqref{eq:future-parameter}, and define
\begin{equation*}
 \bm{V}_t\triangleq 2\lambda \bm{I}_d+
       \sum_{s<t}\nu_{\rm c}(\bm{x}_s^\top\bar{\bm{\theta}}_s)\bm{x}_s\bm{x}_s^\top
       +\frac{e^{-R_s\varepsilon}\sigma b(t)}2\bm{P}_U.
\end{equation*}
\begin{lemma}\label{lem:H_to_V}
With probability at least
$1-\delta$, for every $t_0<t\le T$, we have
\begin{equation*}
 \bm{H}_t\succeq\frac12\bm{V}_t.
\end{equation*}
\end{lemma}

Lemma~\ref{lem:H_to_V} replaces the comparison term in the joint Hessian by an information lower bound, which is more convenient for later analysis.

We next bound the potential sum. To jointly exploit information from both the arm side and the key-term side, we modify the classical elliptical potential lemma and get Theorem~\ref{lem:potential}.
\begin{theorem}\label{lem:potential}
Let $h>0$ and
$\bar{\bm{\theta}}_1,\ldots,\bar{\bm{\theta}}_T\in\Theta$
satisfy $\nu_{\rm c}(\bm{x}_t^\top\bar{\bm{\theta}}_t)\le h$. Then
\[
\sum_{t=1}^T
\log\left(
1+\nu_{\rm c}(\bm{x}_t^\top\bar{\bm{\theta}}_t)
\|\bm{x}_t\|_{\bm{V}_t^{-1}}^2
\right)
\le D_T(h),
\]
where $D_T(h)$ is defined in~Eq.~\eqref{eq:potential_bound}.
The same bound holds for arbitrary coefficients $w_t\in[0,h]$
in place of $\nu_{\rm c}(\bm{x}_t^\top\bar{\bm{\theta}}_t)$, with the same
replacements in $\bm{V}_t$.
\end{theorem}
Combining these results with the regret decomposition similar to \citet{lee2024a} yields the final regret bound. Detailed proofs are given in Appendix~\ref{app:regret}.
\end{proof}

\textbf{Regret order.} Now we give the regret order of HyConUCB. Consider
$b(t)=\lfloor\rho t\rfloor$ with $\rho>0$, and define
$\eta \triangleq \sigma \rho r / 4$. Let $\tilde O$ hide logarithmic factors and constants depending on $S$, $\lambda$, $\varepsilon$, and the GLM coefficients. In this case, our regret bound is 
\begin{equation}\label{eq:rate}
\begin{aligned}
R(T)={}&\tilde O \Bigg(
\underbrace{\sqrt{\frac{d(d-r)T}{\kappa_\star}}+\kappa d(d-r)}_{U^\perp} \\
&+ \underbrace{\sqrt{\frac{drT}{\kappa_\star(1+\eta/\bar v)}} 
+\frac{dr}{\kappa^{-1}+\eta}}_{U}\Bigg).
\end{aligned}
\end{equation}
The two groups represent the contributions from subspace $U^\perp$ and $U$, respectively. Conversations improve the contribution from $U$, while the contribution from $U^\perp$ retains its reward-only order. It is intuitive because the conversational information only covers this subspace.

We derive this regret order by bounding $D_T(h)$. For $h>0$, the second term of $D_T(h)$ satisfies
\begin{equation*}
\begin{aligned}
D_T(h)^{(2)}
\le{}&\left(1+\frac{e^{-R_s\varepsilon}\sigma}{4\lambda}\right)
\frac{rh}{h+2e^{-R_s\varepsilon}\eta} \\
&\times \log\left(
1+\frac{T(h+2e^{-R_s\varepsilon}\eta)}{2\lambda r}
\right)
\end{aligned}
\end{equation*}
The proof is given in Appendix~\ref{app:linear_query}.
Combining this bound with the first term of $D_T(h)$ gives
\begin{equation}\label{eq:linear_schedule_order}
\begin{aligned}
D_T(h)=\tilde O\left((d-r)+r\frac{h}{h+\eta}\right).
\end{aligned}
\end{equation}
Substituting Eq.~\eqref{eq:linear_schedule_order} into Theorem~\ref{thm:regret} gives Eq.~\eqref{eq:rate}.

\textbf{Regret comparisons.}
We compare our bound with the reward-only benchmark and ConDuel. For an explicit comparison, consider the $r$-dimensional subspace $U$ only. The reward-only bound \citep{lee2024a} has a contribution from $U$ of
\(\tilde O\left(\sqrt{\frac{drT}{\kappa_\star}}+\kappa dr\right).\)
Compared with it, our leading term is reduced by a factor of $1/\sqrt{1+\eta/\bar v}$. If $\sigma=\Omega(1/d)$, the lower-order term in~Eq.~\eqref{eq:rate} is $\tilde O(d^2)$, removing $\kappa$ from the reward-only lower-order term. The condition $\sigma=\Omega(1/d)$ is attainable, as illustrated in Appendix~\ref{app:a_good_case}.

Next, we adapt the analysis of ConDuel~\citep{yang2024conversational} to our regret definition and arm-selection rule, retaining its confidence set construction and approach to bounding the potential sum in Eq.~\eqref{eq:challenge}. This gives a contribution from $U$ of
\(\tilde O\left(\kappa \sqrt{\frac{drT}{\kappa_\star}}+\kappa^2 dr\right).\)
Relative to our bound, this bound has an additional factor of $\kappa$ in both its leading and lower-order terms. This factor is from its suboptimal confidence set. Moreover, it does not exhibit the improvement over the reward-only benchmark captured by $\eta$ in~Eq.~\eqref{eq:rate}, while our new potential sum bounding techniques establish this improvement for our algorithm.

%% file: 6-reverse-feedback.tex
\section{Extension to Reverse Configuration}
\label{sec:reverse_feedback}

The hybrid feedback setting considered above can also be extended to reverse configuration: users compare arms and provide absolute feedback on a single key-term at a time. Let $\bm{d}_{a,b}=\bm{x}_{a}-\bm{x}_{b}$ and $\bm{d}_t=\bm{d}_{a_t,a'_t}$. At each round, the learner collects $q_t$ responses $\tilde r_{t,i}\sim p_{\mathrm{c}}(\cdot\mid \bm{z}_{t,i}^\top\ts)$, with $\bm{z}_{t,i}=\tilde {\bm{x}}_{k_{t,i}}$ for $i=1,\ldots,q_t$, then selects $(a_t,a'_t)$ and observes $o_t\sim p_{\mathrm{d}}(\cdot\mid \bm{d}_t^\top\ts)$.

The cumulative negative log-likelihood is
\begin{equation}\label{eq:reverse_loss}
\begin{aligned}
\mathcal L_t(\bm{\theta})
={}&\sum_{s<t}\frac{-o_s\bm{d}_s^\top\bm{\theta}+m_{\mathrm{d}}(\bm{d}_s^\top\bm{\theta})}{a_{\mathrm{d}}}\\
&+\sum_{s\le t}\sum_{i=1}^{q_s}
\frac{-\tilde r_{s,i}\bm{z}_{s,i}^\top\bm{\theta}+m_{\mathrm{c}}(\bm{z}_{s,i}^\top\bm{\theta})}{a_{\mathrm{c}}}.
\end{aligned}
\end{equation}
Using this loss and its Lipschitz constant $L_t$, construct the likelihood-ratio confidence sequence and its ellipsoidal relaxation as in Section ~\ref{sec:confidence}. 
Lemmas~\ref{lem:confidence} and~\ref{lem:ellipsoid} still apply, with appropriate adjustments to the feature-norm bounds for the two feedback types.

For each conversation, apply Eq.~\eqref{eq:query_selection} to the pool of single key-term feature vectors. The argument of Lemma~\ref{lem:information} yields the same form of conditional information lower bound on the span of these vectors, with the absolute-feedback curvature weights.

For arm selection, first retain arms that are not confidently worse than any competitor:
\begin{equation*}
\mathcal P_t=\bigl\{a\in\mathcal A_t:
\bm{d}_{b,a}^\top\hat{\bm{\theta}}_t
\le\gamma_t\|\bm{d}_{b,a}\|_{\bm{H}_t^{-1}},\ \forall b\in\mathcal A_t\bigr\}.
\end{equation*}
Then select the pair with the largest uncertainty in its utility difference:
\begin{equation*}
(a_t, a'_t) \in \argmax_{a,b\in\mathcal P_t}
\|\bm{x}_{a} - \bm{x}_{b}\|_{\bm{H}_t^{-1}}.
\end{equation*}

Following \citet{saha2021optimal}, we define the cumulative dueling regret as
\begin{equation}\label{reverse_regret}
    R_{\mathrm{duel}}(T)=\sum_{t=1}^T\left(\bm{x}_{t,\star}-
\frac{\bm{x}_{a_t}+\bm{x}_{a'_t}}2\right)^\top\ts.
\end{equation}
On the confidence event, an optimal arm belongs to $\mathcal P_t$, so this selection rule bounds the instantaneous regret by $2\gamma_t\|\bm{d}_t\|_{\bm{H}_t^{-1}}$, following the MaxInp analysis of \citet[Theorem~3]{saha2021optimal}. Summing over rounds and applying Cauchy--Schwarz gives
\begin{equation}\label{eq:reverse_regret}
\begin{aligned}
R_{\mathrm{duel}}(T)&=\sum_{t=1}^T\left(\bm{x}_{t,\star}-
\frac{\bm{x}_{a_t}+\bm{x}_{a'_t}}2\right)^\top\ts\\
&\le2\bar\gamma_T\sqrt{T\sum_{t=1}^T\|\bm{d}_t\|_{\bm{H}_t^{-1}}^2}.
\end{aligned}
\end{equation}
We bound the potential sum in Eq.~\eqref{eq:reverse_regret} by adapting \Cref{lem:H_to_V} and \Cref{lem:potential}, with arm differences and single key-term feature vectors replacing arm features and key-term differences, respectively. The proofs follow the same arguments as before and are omitted for brevity.

%% file: 7-syn-experiments.tex
\section{Experiments}

In this section, we describe experimental settings and results on both synthetic data and real-world data. We set the trade-off parameter of HyConUCB to $\xi=0.05$ and use the mean functions $\mu_{\mathrm{c}}(x)=\mu_{\mathrm{d}}(x)=1/(1+e^{-x})$.

\noindent\textbf{Baselines. }
We compare our algorithm with the following baselines:
\begin{itemize}
    \item \textbf{ConDuel}~\citep{yang2024conversational}: ConDuel is designed for conversational dueling bandits, in which both arm-level and key-term-level observations are dueling feedback. To adapt it to our experimental setting, we modify its arm-selection rule to a UCB rule.

    \item \textbf{OFUGLB-UCB}~\citep{lee2024a}: OFUGLB is developed for generalized linear bandits in \citet{lee2024a}. We adapt its original optimism in the face of uncertainty (OFU) rule to a UCB rule, which is more suitable for our conversational bandit setting.

    \item \textbf{Var-Random}: A variant of our algorithm that selects a key-term pair uniformly at random.

    \item \textbf{Var-AM}: A variant of our algorithm that obtains absolute conversational feedback. Its MLE and confidence sequences are adjusted accordingly. For query $i$ in round $t$, it selects
    \[
    k_{t,i}\in\argmax_{k\in\mathcal K}\|\tilde {\bm{x}}_k\|_{\bm{M}_{t,i}^{-1}},
    \]
    where $\bm{M}_{t,i}$ is defined in Section~\ref{sec:algorithm}.
\end{itemize}

\subsection{Experiments on Synthetic Data}
\noindent\textbf{Data Generation. }We construct the synthetic data following \citet{zhang2020conversational}.
We set the dimension to $d=10$, the number of users to $N=100$, the size of the key-term pool to $|\mathcal K|=500$, and the size of the arm pool to $|\mathcal A|=5000$.
We generate the arm feature vectors as follows.
First, for each key-term $k\in\mathcal K$, we sample a pseudo-feature vector $\dot {\bm{x}}_k\in\mathbb R^d$, whose coordinates are drawn independently from the uniform distribution $\mathcal U(-1,1)$.
For each arm $a\in[K]$, we randomly sample an integer $m_a\in\{1,2,3\}$ and then select uniformly without replacement a subset $\mathcal K_a\subseteq\mathcal K$ such that $|\mathcal K_a|=m_a$.
The weights are defined by $W_{a,k}=\frac{1}{m_a}\mathbf 1\{k\in\mathcal K_a\}$.
For each arm $a$, we first draw an unnormalized feature vector from $\bar {\bm{x}}_a\sim\mathcal N\left(\sum_{k\in\mathcal K_a}W_{a,k}\dot {\bm{x}}_k,\sigma_x^2\bm{I}_d\right)$, where $\sigma_x=0.02$, and then normalize it as $\bm{x}_a=\frac{\bar {\bm{x}}_a}{\|\bar {\bm{x}}_a\|_2}$.
Thus, every arm satisfies $\|\bm{x}_a\|_2=1$.
After generating all arm features, we compute the feature vector of each key-term $k$ as $\tilde {\bm{x}}_k=\frac{\sum_{a\in\mathcal A}W_{a,k}\bm{x}_a}{\sum_{a\in\mathcal A}W_{a,k}}$.
Finally, for each user $u\in[N]$, we independently sample $\bm{g}_u\sim\mathcal N(\bm{0},\bm{I}_d)$ and define the user parameter as $\bm{\theta}_u^\star=S\frac{\bm{g}_u}{\|\bm{g}_u\|_2}$ for $S\in\{4,6,8\}$.
In our simulations, we set the time horizon to $T=5000$ and the conversation-frequency function to $b(t)=10\left\lfloor\frac{t}{50}\right\rfloor$.
At each round $t$, we randomly select 100 arms from $\mathcal A$ without replacement to form $\mathcal A_t$.
The same $\mathcal A_t$ is presented to all algorithms.

\noindent\textbf{Cumulative Regret. }
We repeat each experiment over $10$ random seeds, averaging over all users within each run. As shown in Fig.~\ref{fig:synthetic_regret_S}, HyConUCB reduces the final mean cumulative regret by more than $40\%$ relative to OFUGLB-UCB, the reward-only baseline, for every tested value of $S$. HyConUCB achieves the lowest final mean regret for $S=4$ and $S=6$, while Var-Random performs slightly better for $S=8$. As $S$ increases, ConDuel's advantage over OFUGLB-UCB narrows. This pattern is consistent with ConDuel's sensitivity to the worst-case curvature, which can become small as the parameter norm grows.

\begin{figure}[H]
  \centering
  \includegraphics[width=\linewidth]{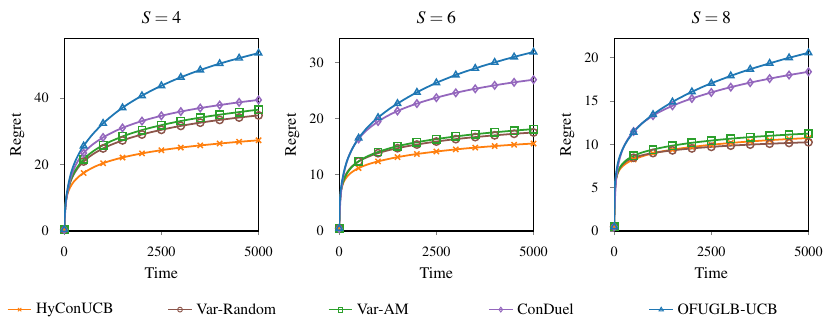}
  \caption{Cumulative regret on synthetic data for different values of the parameter norm bound $S$. The advantage of HyConUCB and its variants over other baselines becomes more distinct as $S$ increases, consistent with our theoretical analysis.}
  \label{fig:synthetic_regret_S}
\end{figure}

\noindent\textbf{Parameter Estimation Error. }We measure estimation accuracy by the mean error $\|\hat {\bm{\theta}}_t-\ts\|_2$, with smaller values indicating more accurate estimates. Fig.~\ref{fig:synthetic_parameter_error_S} shows that the error generally decreases over time for all algorithms. HyConUCB achieves the lowest final mean error for $S=4$ and $S=6$ and remains competitive for $S=8$, where Var-Random achieves the lowest error and Var-AM also finishes slightly below HyConUCB. Across all three settings, HyConUCB yields more accurate estimates than OFUGLB-UCB, demonstrating the benefit of conversational feedback for parameter estimation.

\begin{figure}[H]
  \centering
  \includegraphics[width=\linewidth]{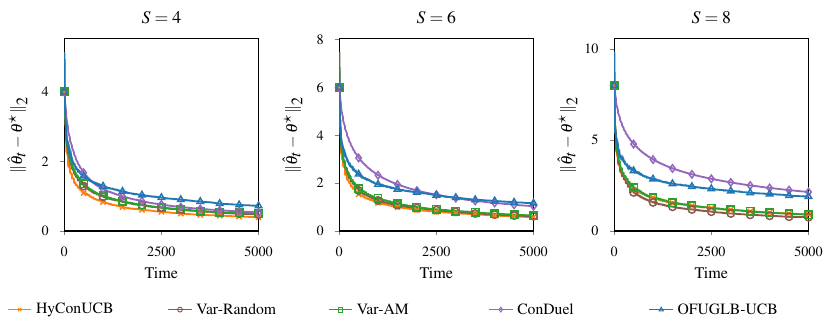}
  \caption{Parameter estimation error on synthetic data for different values of $S$. Lower values indicate more accurate estimates. HyConUCB and its variants generally achieve lower estimation errors than other baselines.}
  \label{fig:synthetic_parameter_error_S}
\end{figure}

\noindent\textbf{Advantage of Hybrid Feedback. }To examine the benefit of hybrid feedback, we use the common intercept construction in Appendix~\ref{app:a_good_case}. 
More concretely, we first generate the relative arm features $\bm{u}_a\in\mathbb R^r$ and user directions $\bm{v}_u\in\mathbb R^r$ using the sampling procedure above and normalize both to unit length, so that $\|\bm{u}_a\|_2=\|\bm{v}_u\|_2=1$.
We then set $d=r+1$ and construct $\bm{x}_a=\frac{1}{\sqrt{2}}(1,\bm{u}_a)$ and $\bm{\theta}_u^\star=\sqrt{2}(B,C\bm{v}_u)$.
Consequently, we have $S=\sqrt{2(B^2+C^2)}$ and $\bm{x}_a^\top\bm{\theta}_u^\star=B+C\,\bm{u}_a^\top \bm{v}_u$.
We set $r=10$, $d=11$, $C=1$, and vary $B\in\{2,4,6\}$. The common intercept $B$ cancels out in key-term comparisons, making relative feedback more informative than absolute responses in this setting.

Fig.~\ref{fig:synthetic_regret_B} shows that HyConUCB achieves the lowest final mean cumulative regret for all three values of $B$. ConDuel also outperforms OFUGLB-UCB, consistent with the value of relative conversational feedback in this setting. As $B$ increases, Var-AM's relative advantage over OFUGLB-UCB diminishes, and their regret curves nearly coincide at $B=6$. While Var-AM performs similarly to HyConUCB in the preceding synthetic experiments, it falls further behind HyConUCB here. This larger performance gap provides further evidence for the benefit of hybrid feedback in this setting.

\begin{figure}[H]
  \centering
  \includegraphics[width=\linewidth]{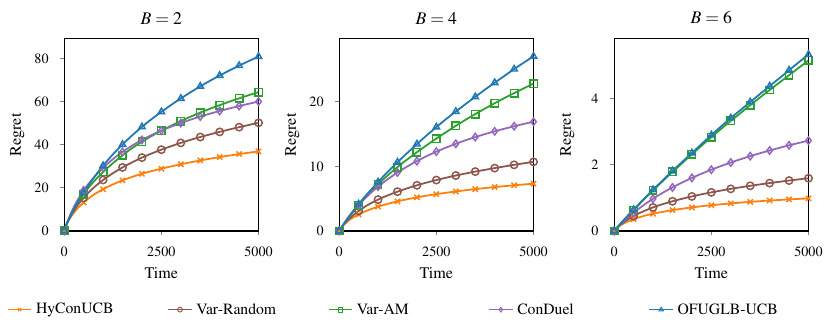}
  \caption{Cumulative regret on synthetic data for different values of the common intercept $B$. As $B$ increases, relative feedback becomes more informative than absolute feedback. The deteriorating performance of Var-AM is consistent with this trend.}
  \label{fig:synthetic_regret_B}
\end{figure}

%% file: 8-real-experiments.tex
\subsection{Experiments on the Yelp Dataset}
\noindent\textbf{Data Generation.} We conduct experiments on the Yelp dataset, which we preprocess following \citet{zhang2020conversational}.
We retain restaurants with at least \(50\) reviews and users with at least \(100\) reviews, and then randomly select \(100\) users for the simulations.
Each restaurant is treated as an arm, and its associated categories are treated as key-terms.
The \(500\) key-terms with the largest restaurant coverage are retained.
We construct restaurant feature vectors from the city, categories, star rating, review count, and \(34\) additional common business attributes, and then apply PCA to obtain \(10\)-dimensional arm feature vectors.
Finally, we convert the original five-point ratings into binary feedback.
For each user \(u\), we estimate the preference parameter by solving the following
\(\ell_2\)-ball-constrained least-squares problem:
\(\bm{\theta}_u^\star
=
\arg\min_{\|\bm{\theta}\|_2\le S}
\sum_{a\in\mathcal I_u}
\left(y_{u,a}-\bm{x}_a^\top\bm{\theta}\right)^{2}.
\)
We set \(S=2\) because preliminary unconstrained estimates show that
\(93\) of the \(100\) users have parameter norms below \(2\).
Both arm and conversational feedback are generated using \(\bm{\theta}_u^\star\).
For each user, we run \(T=5{,}000\) rounds and compute cumulative regret.

\noindent\textbf{Results.} We repeat this experiment \(10\) times, and Fig.~\ref{fig:yelp_regret} shows the mean cumulative regret. HyConUCB achieves the lowest cumulative regret on the Yelp dataset. ConDuel performs second best, followed by Var-AM and Var-Random, whereas OFUGLB-UCB has the highest cumulative regret.

\begin{figure}[H]
  \centering
  \begingroup
    \phantomsubcaption\label{fig:yelp_regret}
    \phantomsubcaption\label{fig:llm_oracle_regret}
    \phantomsubcaption\label{fig:llm_expert_regret}
  \endgroup
  \includegraphics[width=\linewidth]{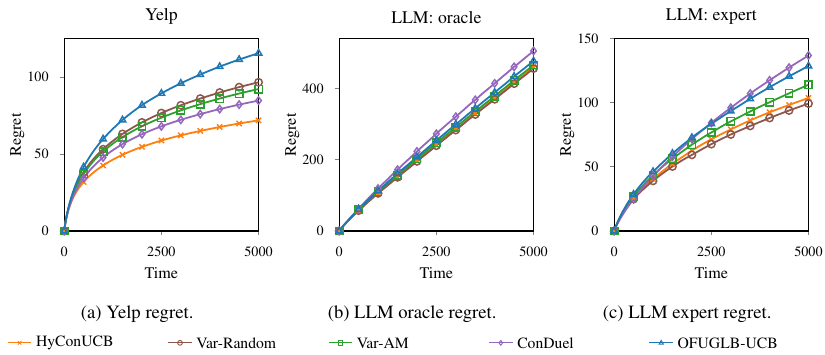}
  \caption{Results on real-world data. For LLM routing, oracle regret in panel (b) grows approximately linearly, making differences between algorithms less apparent. We therefore use expert regret in panel (c) as our primary metric.}
\end{figure}

\subsection{Experiments on LLM Routing}
In large language model (LLM) routing systems, a router selects an LLM from the candidate pool to answer a query. We formulate LLM routing as a conversational bandit problem, treating candidate LLMs as arms and task categories as key-terms.

\noindent\textbf{Dataset.} We conduct our experiments on the xRouteBench dataset~\citep{feng2026llmroute}, using its \texttt{llmrouter\_generic} configuration. This configuration contains $13$ tasks and $18$ candidate LLMs. For each question, the dataset records the performance score of every candidate LLM. We divide the data into a $60\%$ training split, a $20\%$ validation split, and a $20\%$ test split. The validation split is used for preliminary evaluation and hyperparameter selection. 
We select $4$ tasks and $4$ models. Hence, $P=K=4$. Our choice of $K=4$ is consistent with the small model pools used in prior LLM routing studies, including RouteLLM with two models~\citep{ong2024routellm}, MetaLLM with four models~\citep{nguyen2025metallm}, and greedy LinUCB with six models~\citep{poon2026online}. 
For each selected task, we designate the model with the highest accuracy on that task as its expert. The selected tasks and their expert models are ARC-Challenge (Qwen3-Coder-Next), GSM8K (Llama-4-Maverick), MATH (Mistral-Small-3-24B-Instruct), and MMLU (Llama-3.3-70B-Instruct-Turbo).

\noindent\textbf{Feature and Feedback Construction.}
We construct the feature vectors as follows. First, we use \texttt{BAAI/bge-small-en-v1.5} to obtain a query embedding. 
We then train a classifier using data from the training split. The classifier takes a query embedding as input and outputs a probability vector $\bm{p}_t\in\mathbb{R}^4$, whose components represent the predicted probabilities for the four task categories. 
We append a constant coordinate of $1$ to $\bm{p}_t$ to obtain $\bm{h}_t\in\mathbb{R}^5$. For each model $k$, we define $\bm{v}_{t,k}=\bm{e}_k\otimes \bm{h}_t$. Finally, we normalize this vector and append a global intercept, yielding $\bm{x}_{t,k}\in \mathbb{R}^{21}$. Hence, $d=21$.

For arm feedback, we use the logged outcomes in the dataset. We associate each task with its designated expert model. Thus the weight matrix $\bm{W}=\bm{I}_4$. Because the dataset does not contain conversational feedback records, we generate this feedback synthetically. We first estimate the user parameter $\ts$ via maximum likelihood using the training split. It should be noted that $\ts$ is used exclusively to generate conversational feedback, while arm-side rewards are taken directly from the dataset's logged model outcomes. Because the resulting estimate has norm $\|\ts\|_2\approx 5.504$, we set $S=6$. During the online experiment, we generate conversational feedback according to Eq.~\eqref{eq:dueling_feedback}. We run each experiment for $T=5{,}000$ rounds on the test split and set the conversation-frequency function $b(t)=\lfloor t/5\rfloor$.

\noindent\textbf{Metrics.} Let $Q_t$ denote the query at round $t$, and let $r_{Q_t,k}\in\{0,1\}$ indicate whether model $k$ answers it correctly. Define the cumulative oracle regret
\begin{equation*}
    R_T=\sum_{t=1}^{T}
\left(\max_k r_{Q_t,k}-r_{Q_t,a_t}\right). 
\end{equation*}
However, it can grow linearly even when the task expert is always selected, because another model may answer a query correctly when the task expert answers it incorrectly. So we take the cumulative expert regret $R_T^{\mathrm{expert}}$ as our primary metric.
Let $k^{\mathrm{expert}}(Q_t)$ be the expert model on query $Q_t$. Define
\begin{equation*}
    R_T^{\mathrm{expert}}=
    \sum_{t=1}^{T}
    \left(
    r_{Q_t,k^{\mathrm{expert}}(Q_t)}-r_{Q_t,a_t}
    \right).    
\end{equation*}
This metric measures the gap between the selected model and the task expert.

\noindent\textbf{Results.}
We repeat the experiment over $200$ random seeds and report the mean cumulative oracle and expert regret in Fig.~\ref{fig:llm_oracle_regret} and Fig.~\ref{fig:llm_expert_regret}, respectively.
HyConUCB and its variants outperform other baselines in terms of final mean cumulative regret, with Var-Random ranking first, followed closely by HyConUCB and then Var-AM. 
ConDuel incurs the highest regret, slightly higher than that of OFUGLB-UCB.

%% file: 9-conclusion.tex
\section{Concluding Remarks}
\label{sec:conclusion}

In this paper, we formulate the conversational generalized linear bandits with hybrid feedback problem. 
We develop time-uniform confidence sequences for heterogeneous arm and key-term feedback and propose the HyConUCB algorithm. 
Our theoretical analysis shows that HyConUCB can achieve a tighter regret bound than the reward-only benchmark and prior conversational generalized linear bandit methods. 
Experiments on synthetic and real-world datasets further demonstrate its effectiveness.

Several directions remain open. From the hybrid feedback perspective, we consider a setting in which arm interactions and key-term queries use different feedback types. A natural extension is a fully hybrid setting, in which both channels support absolute and relative feedback and the learner can adaptively select the feedback type.

From the conversation perspective, we assume that conversations incur no cost and schedule them according to a fixed frequency function. In practice, however, conversations may incur costs. Developing algorithms that account for these costs and adaptively decide when to conduct conversations is therefore a promising direction for future work.

%% file: appendix/1-confidence-set.tex
\section{Auxiliary Results}
\label{app:confidence_localization}

\subsection{Ellipsoidal Confidence Sequence}\label{app:confidence_set}
\begin{proof}[Proof of Lemma~\ref{lem:ellipsoid}]
Since $\Theta$ is convex and
$\widehat{\bm{\theta}}_t$ minimizes $\mathcal L_t$ over $\Theta$, we have for $\bm{\theta}\in\Theta$, 
\[
\nabla\mathcal L_t(\widehat{\bm{\theta}}_t)^\top
 (\bm{\theta}-\widehat{\bm{\theta}}_t)\ge0.
\]
Every observation feature $\bm{v}$ in $\mathcal L_t$ satisfies
$\|\bm{v}\|_2\le2$, and hence
$|\bm{v}^\top(\bm{\theta}-\widehat{\bm{\theta}}_t)|\le4S$.
By Lemma~D.1 in \citet{lee2024a}, we have
\begin{align*}
\int_0^1(1-s)\dot\mu\bigl(\bm{v}^\top
 ((1-s)\widehat{\bm{\theta}}_t+s\bm{\theta})\bigr)\,ds
 \ge\frac{\dot\mu(\bm{v}^\top\widehat{\bm{\theta}}_t)}{2(1+2SR_s)}.
\end{align*}
By Taylor's formula and the first-order optimality condition, we obtain
\begin{align*}
\beta_t^2
&\ge\mathcal L_t(\bm{\theta})-\mathcal L_t(\widehat{\bm{\theta}}_t)\\
&\ge\int_0^1(1-s)\|\bm{\theta}-\widehat{\bm{\theta}}_t\|^2_{
 \nabla^2\mathcal L_t((1-s)\widehat{\bm{\theta}}_t+s\bm{\theta})}\,ds\\
&\ge\frac{\|\bm{\theta}-\widehat{\bm{\theta}}_t\|^2_{
 \nabla^2\mathcal L_t(\widehat{\bm{\theta}}_t)}}
 {2(1+2SR_s)}.
\end{align*}
Adding the regularization term and using
$\|\bm{\theta}-\widehat{\bm{\theta}}_t\|_2\le2S$
therefore gives
\begin{align*}
\|\bm{\theta}-\widehat{\bm{\theta}}_t\|_{\bm{H}_t}^2
&\le2(1+2SR_s)\beta_t^2+4\lambda S^2\\
&\le2(1+2SR_s)(\beta_t^2+4\lambda S^2)=\gamma_t^2.
\end{align*}
Thus $\mathcal C_t\subseteq\mathcal E_t$, which completes the proof.
\end{proof}

\subsection{Guarantees for the Query Selection Rule}
\label{app:algorithm_proofs}

Suppress the query indices and write $A(\bm{p})=\sum_e p_es_e$ and
$j_{\rm geo}=\max_{\bm{p}\in\Delta_{N_q}}j(\bm{p})>0$.
Let $\bm{p}^\star$ solve~Eq.~\eqref{eq:query_selection}.
For $e_{\max}\in\arg\max_e s_e$ and
$\bm{p}_\tau=(1-\tau)\bm{p}^\star+\tau\bm{\delta}_{e_{\max}}$, we have
$j(\bm{p}_\tau)\ge(1-\tau)j(\bm{p}^\star)$.
Optimality of $\bm{p}^\star$ therefore gives
\[
0\ge\tau\{1-A(\bm{p}^\star)\}+\xi\log(1-\tau).
\]
Dividing by $\tau$ and letting $\tau\downarrow0$ yields
\begin{equation}\label{eq:app_maxinp_guarantee}
A(\bm{p}^\star)\ge1-\xi.
\end{equation}
Thus, the expected MaxInp score under $\bm{p}^\star$ is at least $1-\xi$.

Next, take $\bm{p}^{\rm geo}\in\arg\max_{\bm{p}\in\Delta_{N_q}}j(\bm{p})$ and set
$\bm{p}_\tau=(1-\tau)\bm{p}^\star+\tau \bm{p}^{\rm geo}$.
Since $A(\bm{p})\in[0,1]$ and $j$ is concave, optimality gives
\[
0\ge-\tau+\xi\log\left(
1+\tau\left\{\frac{j_{\rm geo}}{j(\bm{p}^\star)}-1\right\}\right).
\]
Again dividing by $\tau$ and letting $\tau\downarrow0$ gives
\begin{equation}\label{eq:app_coverage_guarantee}
j(\bm{p}^\star)\ge\frac{\xi}{1+\xi}j_{\rm geo}.
\end{equation}
Equations~\eqref{eq:app_maxinp_guarantee} and~Eq.~\eqref{eq:app_coverage_guarantee} formalize the balance in the query rule: the first preserves the MaxInp score, while the second guarantees coverage of the query subspace. Next, we prove Lemma~\ref{lem:information}.

Let $c_\star=\min_e\nu_{\mathrm{d}}(\bm{q}_e^\top\ts)>0$, and let
$\mathcal H$ be the history immediately before the current query,
including earlier queries in the same round. Then
\begin{align*}
\mathbb E\!\left[
\nu_{\mathrm{d}}(\tilde {\bm{d}}^\top\ts)
\tilde {\bm{d}}\tilde {\bm{d}}^\top\mid\mathcal H\right]
&=\sum_e p_e^\star\nu_{\mathrm{d}}(\bm{q}_e^\top\ts)\bm{q}_e\bm{q}_e^\top\\
&\succeq c_\star \bm{V}(\bm{p}^\star)
\succeq\frac{\xi}{1+\xi}c_\star j_{\rm geo}\bm{P}_U.
\end{align*}
Thus each query satisfies the information bound with
\begin{equation}\label{eq:app_query_sigma}
\sigma=\frac{\xi}{1+\xi}c_\star j_{\rm geo}.
\end{equation}
Summing over the $q_t$ queries proves
Lemma~\ref{lem:information}.

\subsection{Statistical Guarantees}
\label{sec:statistical_grarantees}

\begin{lemma}\label{lem:jensen}
For $j\in\{\mathrm{c},\mathrm{d}\}$ and any $u,v$,
\[
\frac{m_j(u)+m_j(v)}2-m_j\left(\frac{u+v}2\right)
\ge\frac{\dot\mu_j(u)(u-v)^2}{8(1+R_s|u-v|/2)^2}.
\]
\end{lemma}

\begin{proof}
For $s\in[0,1]$, self-concordance gives
\[
m_j''(u+s(v-u))\ge\dot\mu_j(u)e^{-sR_s|u-v|}.
\]
Taylor's formula therefore gives
\begin{align*}
&\frac{m_j(u)+m_j(v)}2-m_j\left(\frac{u+v}2\right)\\
&\quad=(u-v)^2\biggl(
\int_0^{1/2}\frac{s}{2}m_j''(u+s(v-u))\,ds
\\&\qquad\quad+
\int_{1/2}^1\frac{1-s}{2}m_j''(u+s(v-u))\,ds\biggr)\\
&\quad\ge\frac{\dot\mu_j(u)}{2R_s^2}
\left(1-e^{-R_s|u-v|/2}\right)^2\\
&\quad\ge\frac{\dot\mu_j(u)(u-v)^2}{8(1+R_s|u-v|/2)^2}.
\end{align*}
The last step uses $e^{-x/2}\le(1+x/2)^{-1}$.
\end{proof}

\begin{lemma}\label{lem:fisher}
Let
\[
\bm{F}_t^q=\sum_{s=1}^t\sum_{i=1}^{q_s}
\nu_{\mathrm{d}}(\tilde {\bm{d}}_{s,i}^\top\ts)
\tilde {\bm{d}}_{s,i}\tilde {\bm{d}}_{s,i}^\top.
\]
For every $t\le T$,
\[
\mathbb P\left(\bm{F}_t^q\not\succeq\frac{\sigma b(t)}2\bm{P}_U\right)
\le r\exp\left(-\frac{a_{\mathrm{d}}\sigma b(t)}{32q_{\max}R_{\dot\mu}}\right).
\]
Consequently, with probability at least $1-\delta/2$,
\begin{gather*}
\bm{F}_t^q\succeq\frac{\sigma b(t)}2\bm{P}_U
\quad\text{for all }t\le T\text{ with }\\
b(t)\ge\frac{32q_{\max}R_{\dot\mu}}{a_{\mathrm{d}}\sigma}
\log\frac{2rT}{\delta}.
\end{gather*}
\end{lemma}

\begin{proof}
Set $\bm{F}_0^q=\bm{0}$ and restrict matrices to $U$, an $r$-dimensional space.
The round increments $\bm{X}_s=(\bm{F}_s^q-\bm{F}_{s-1}^q)|_U$ satisfy
\[
\bm{0}\preceq \bm{X}_s\preceq\frac{4q_{\max}R_{\dot\mu}}{a_{\mathrm{d}}}\bm{I}_U,
\qquad
\mathbb E[\bm{X}_s\mid\mathcal F_{s-1}]\succeq\sigma q_s\bm{I}_U,
\]
by $\|\tilde {\bm{d}}_{s,i}\|_2\le2$, $\nu_{\mathrm{d}}\le R_{\dot\mu}/a_{\mathrm{d}}$,
and Lemma~\ref{lem:information}.
Hence the conditional expectations sum $\sum_{s=1}^t
\mathbb E[\bm{X}_s\mid\mathcal F_{s-1}]$ is at least $\sigma b(t)\bm{I}_U$.
Then by adapted matrix Chernoff inequality
\citep[Theorem~3.1]{tropp2011user}, with relative
deviation $1/2$ and the bounds above, we have
\[
\mathbb P\left(\bm{F}_t^q\not\succeq\frac{\sigma b(t)}2\bm{P}_U\right)
\le r\exp\left(-\frac{a_{\mathrm{d}}\sigma b(t)}{32q_{\max}R_{\dot\mu}}\right).
\]
Here all query directions lie in $U$, so the restriction controls exactly the claimed lower bound. 
Under the stated threshold, the right-hand side is at most $\delta/(2T)$. A union bound on $T$ proves the simultaneous conclusion.
\end{proof}

Let $\mathcal E$ be the intersection of the simultaneous events in Lemmas~\ref{lem:confidence} and~\ref{lem:fisher},
so that $\mathbb P(\mathcal E)\ge1-\delta$.

\begin{lemma}\label{lem:localization}
On $\mathcal E$, for every $t_0<t\le T$,
\[
\sup_{\bm{\theta}\in\mathcal C_t}
\sup_{\substack{\bm{v}\in U\\\|\bm{v}\|_2\le2}}
|\bm{v}^\top(\bm{\theta}-\ts)|\le\varepsilon.
\]
\end{lemma}

\begin{proof}
Fix $t>t_0$ and $\bm{\theta}\in\mathcal C_t$ on $\mathcal E$.
Monotonicity of $b$ gives $b(t)\ge n_0$, so the bound in Lemma~\ref{lem:fisher} applies.
Since $\bm{\theta},\ts\in\mathcal C_t$ and their midpoint lies in
$\Theta$, optimality of $\hat{\bm{\theta}}_t$ gives
\[
\frac{\mathcal L_t(\bm{\theta})+\mathcal L_t(\ts)}2
-\mathcal L_t\left(\frac{\bm{\theta}+\ts}2\right)
\le\beta_t^2.
\]
The linear terms in the losses cancel. The remaining terms are
nonnegative by convexity of $m_{\mathrm{c}}$ and $m_{\mathrm{d}}$. Keeping the query terms gives
\begin{equation*}
 \frac1{a_{\mathrm{d}}}\sum_{s\le t}\sum_{i=1}^{q_s}\left[
 \frac{m_{\mathrm{d}}(\tilde {\bm{d}}_{s,i}^\top\bm{\theta})
       +m_{\mathrm{d}}(\tilde {\bm{d}}_{s,i}^\top\ts)}2
 -m_{\mathrm{d}}\left(\frac{\tilde {\bm{d}}_{s,i}^\top(\bm{\theta}+\ts)}2\right)
 \right]\le\beta_t^2.
\end{equation*}
Every query satisfies
$|\tilde {\bm{d}}_{s,i}^\top(\bm{\theta}-\ts)|\le u$,
where $u=2\|\bm{P}_U(\bm{\theta}-\ts)\|_2$.
By Lemmas~\ref{lem:jensen} and~\ref{lem:fisher},
\begin{equation*}
 \beta_t^2
 \ge\frac{(\bm{\theta}-\ts)^\top \bm{F}_t^q(\bm{\theta}-\ts)}
 {8(1+R_su/2)^2}
 \ge\frac{\sigma b(t)}{64}
       \frac{u^2}{(1+R_su/2)^2}.
\end{equation*}
The definition of $n_0$ gives
\[
 b(t)\ge n_0
 \ge\frac{64(1+\max_{s\le T}\beta_s^2)}{\sigma}
       \left(\frac1\varepsilon+\frac{R_s}2\right)^2.
\]
Combining these bounds,
and using $\beta_t^2\le1+\max_{s\le T}\beta_s^2$, yields
\[
 \frac{u}{1+R_su/2}\le\frac{\varepsilon}{1+R_s\varepsilon/2}.
\]
The function $u/(1+R_su/2)$ is increasing for $u\ge0$, so
$u\le\varepsilon$. For every $\bm{v}\in U$ with $\|\bm{v}\|_2\le2$,
\[
 |\bm{v}^\top(\bm{\theta}-\ts)|
 =|\bm{v}^\top \bm{P}_U(\bm{\theta}-\ts)|
 \le2\|\bm{P}_U(\bm{\theta}-\ts)\|_2
 =u\le\varepsilon.
\]
\end{proof}

%% file: appendix/2-regret-analysis-theorem.tex
\section{The Regret Analysis}
\label{app:regret}

\begin{proof}[Proof of~\Cref{thm:regret}]
Work on $\mathcal E$, which is defined in Section~\ref{sec:statistical_grarantees}. Following
\citet[Appendix~E]{lee2024a}, choose
\begin{equation}\label{eq:future-parameter}
\bar{\bm{\theta}}_t\in
\argmin_{\bm{\theta}\in\bigcup_{\ell=t}^T\mathcal C_\ell}
\dot\mu_{\mathrm{c}}(\bm{x}_t^\top\bm{\theta}).
\end{equation}
Let $\Delta_t=\mu_{\mathrm{c}}((\bm{x}_{t,\star})^\top\ts)
-\mu_{\mathrm{c}}(\bm{x}_t^\top\ts)$, and define
\[
\mathcal T=\left\{t_0<t\le T:
\|\bm{x}_t\|_{\bm{V}_t^{-1}}\le\frac1{4G_T}\right\},
\qquad R_{\mathcal T}=\sum_{t\in\mathcal T}\Delta_t.
\]
Lemma~\ref{lem:big_round_num} and $\Delta_t\le2R_{\mu,\star}$ give
\begin{equation}\label{eq:proof-outside-rounds}
R(T)\le R_{\mathcal T}+2R_{\mu,\star}t_0
+\frac{32R_{\mu,\star}a_{\mathrm{c}}\kappa G_T^2}{\log2}
D_T(1/(a_{\mathrm{c}}\kappa)).
\end{equation}

By the confidence ellipsoid and UCB selection,
$(\bm{x}_{t,\star})^\top\ts
\le \bm{x}_t^\top\hat{\bm{\theta}}_t+\gamma_t\|\bm{x}_t\|_{\bm{H}_t^{-1}}$.
Thus Lemma~\ref{lem:H_to_V} gives, for $t>t_0$,
\[
0\le(\bm{x}_{t,\star}-\bm{x}_t)^\top\ts
\le2\gamma_t\|\bm{x}_t\|_{\bm{H}_t^{-1}}
\le2\sqrt2\bar\gamma_T\|\bm{x}_t\|_{\bm{V}_t^{-1}}.
\]
On $\mathcal T$, we have
\begin{equation*}
 R_s(\bm{x}_{t,\star}-\bm{x}_t)^\top\ts
 \le2\sqrt2R_s\bar\gamma_T\|\bm{x}_t\|_{\bm{V}_t^{-1}}
 \le\frac1{\sqrt2}\leq\frac{3}{4}.
\end{equation*}
Integrating Lemma~\ref{lem:dot_mu_to_exp} over this gap yields
\[
\Delta_t
=\int_{\bm{x}_t^\top\ts}^{(\bm{x}_{t,\star})^\top\ts}
\dot\mu_{\mathrm{c}}(u)\,du
\le2\sqrt2e^{3/4}\bar\gamma_T
\dot\mu_{\mathrm{c}}(\bm{x}_t^\top\ts)\|\bm{x}_t\|_{\bm{V}_t^{-1}}.
\]

For $t\in\mathcal T$, choose $\ell_t\ge t$ with
$\bar{\bm{\theta}}_t\in\mathcal C_{\ell_t}$.
Both $\bar{\bm{\theta}}_t$ and $\ts$ belong to this confidence set,
so $\|\bar{\bm{\theta}}_t-\ts\|_{\bm{H}_{\ell_t}}\le2\gamma_{\ell_t}$.
Since $\bm{V}_{\ell_t}\succeq \bm{V}_t$, Lemma~\ref{lem:H_to_V} implies
\[
|\bm{x}_t^\top(\bar{\bm{\theta}}_t-\ts)|
\le2\gamma_{\ell_t}\|\bm{x}_t\|_{\bm{H}_{\ell_t}^{-1}}
\le2\sqrt2\bar\gamma_T\|\bm{x}_t\|_{\bm{V}_t^{-1}}
\le\frac{\bar\gamma_T}{\sqrt2G_T}.
\]
Using $R_s\bar\gamma_T\le G_T$ and
Lemma~\ref{lem:dot_mu_to_exp}, we obtain
\[
\dot\mu_{\mathrm{c}}(\bm{x}_t^\top\bar{\bm{\theta}}_t)
\ge e^{-3/4}\dot\mu_{\mathrm{c}}(\bm{x}_t^\top\ts).
\]
Also, self-concordance gives, for $z\le z^\star$,
\[
\dot\mu_{\mathrm{c}}(z)-\dot\mu_{\mathrm{c}}(z^\star)
\le R_s\int_z^{z^\star}\dot\mu_{\mathrm{c}}(u)\,du
=R_s\{\mu_{\mathrm{c}}(z^\star)-\mu_{\mathrm{c}}(z)\}.
\]
Consequently,
\begin{align*}
a_{\mathrm{c}}\sum_{t\in\mathcal T}\nu_{\mathrm{c}}(\bm{x}_t^\top\bar{\bm{\theta}}_t)
&\le\sum_{t\in\mathcal T}\dot\mu_{\mathrm{c}}(\bm{x}_t^\top\ts)\\
&\le\sum_{t\in\mathcal T}
\dot\mu_{\mathrm{c}}((\bm{x}_{t,\star})^\top\ts)+R_sR_{\mathcal T}
\\&\le T/\kappa_\star+R_sR_{\mathcal T}.
\end{align*}

Combining these bounds and applying Cauchy--Schwarz gives
\begin{align*}
R_{\mathcal T}
&\le2\sqrt2e^{3/2}a_{\mathrm{c}}\bar\gamma_T
\sum_{t\in\mathcal T}\nu_{\mathrm{c}}(\bm{x}_t^\top\bar{\bm{\theta}}_t)
\|\bm{x}_t\|_{\bm{V}_t^{-1}}\\
&\le2\sqrt2e^{3/2}\sqrt{a_{\mathrm{c}}}\bar\gamma_T
\sqrt{T/\kappa_\star+R_sR_{\mathcal T}}\sqrt{\sum_{t\in\mathcal T}\nu_{\mathrm{c}}(\bm{x}_t^\top\bar{\bm{\theta}}_t)
\|\bm{x}_t\|_{\bm{V}_t^{-1}}^2}\\
&\le4e^{3/2}\sqrt{a_{\mathrm{c}}}\bar\gamma_T
\sqrt{D_T(\bar v/a_{\mathrm{c}})(T/\kappa_\star+R_sR_{\mathcal T})}.
\end{align*}
The last step uses Theorem~\ref{lem:potential} with $h=\bar v/a_{\mathrm{c}}$
and $u\le2\log(1+u)$ for $0\le u\le1$;
indeed, on $\mathcal T$,
\[
 u=\nu_{\mathrm{c}}(\bm{x}_t^\top\bar{\bm{\theta}}_t)\|\bm{x}_t\|_{\bm{V}_t^{-1}}^2
 \le\frac{R_{\dot\mu}/a_{\mathrm{c}}}{16G_T^2}\le\frac1{16}.
\]
Applying $\sqrt{xy}\leq\frac{x+y}{2}$ yields
\[
R_{\mathcal T}
\le8e^{3/2}\sqrt{a_{\mathrm{c}}}\bar\gamma_T
\sqrt{D_T(\bar v/a_{\mathrm{c}})T/\kappa_\star}
+16e^3a_{\mathrm{c}}R_s\bar\gamma_T^2D_T(\bar v/a_{\mathrm{c}}).
\]
Combining this with~Eq.~\eqref{eq:proof-outside-rounds}
proves~Eq.~\eqref{eq:main_bound}.
\end{proof}

\begin{proof}[Proof of Lemma~\ref{lem:H_to_V}]
For $s<t$, the definition of $\bar{\bm{\theta}}_s$ and
$\hat{\bm{\theta}}_t\in\mathcal C_t$ give
$\nu_{\mathrm{c}}(\bm{x}_s^\top\hat{\bm{\theta}}_t)
\ge\nu_{\mathrm{c}}(\bm{x}_s^\top\bar{\bm{\theta}}_s)$.
For $t>t_0$, Lemma~\ref{lem:localization} gives
$|\tilde {\bm{d}}_{s,i}^\top(\hat{\bm{\theta}}_t-\ts)|
\le\varepsilon$. So, by Lemma~\ref{lem:dot_mu_to_exp},
\begin{equation*}
 \nu_{\mathrm{d}}(\tilde {\bm{d}}_{s,i}^\top\hat{\bm{\theta}}_t)
 \ge e^{-R_s\varepsilon}
       \nu_{\mathrm{d}}(\tilde {\bm{d}}_{s,i}^\top\ts).
\end{equation*}
Hence, by Lemma~\ref{lem:fisher},
\begin{align*}
\bm{H}_t&\succeq\lambda \bm{I}_d+
\sum_{s<t}\nu_{\mathrm{c}}(\bm{x}_s^\top\bar{\bm{\theta}}_s)\bm{x}_s\bm{x}_s^\top
+e^{-R_s\varepsilon}\bm{F}_t^q\\
&\succeq\lambda \bm{I}_d+
\sum_{s<t}\nu_{\mathrm{c}}(\bm{x}_s^\top\bar{\bm{\theta}}_s)\bm{x}_s\bm{x}_s^\top
+\frac{e^{-R_s\varepsilon}\sigma b(t)}2\bm{P}_U
\succeq\frac12\bm{V}_t.
\end{align*}
\end{proof}

\begin{proof}[Proof of~\Cref{lem:potential}]
We prove the general situation for arbitrary $w_t\in[0,h]$ and
$\bm{V}_t=2\lambda \bm{I}_d+\sum_{s<t}w_s\bm{x}_s\bm{x}_s^\top+\frac{e^{-R_s\varepsilon}\sigma}{2}b(t)\bm{P}_U$.
Set $\bm{J}_0=2\lambda \bm{I}_d$ and $\bm{J}_t=\bm{V}_t+w_t\bm{x}_t\bm{x}_t^\top$.
By the determinant lemma and telescoping, we can get
\begin{align*}
\Phi_T&:=\sum_{t=1}^T\log(1+w_t\|\bm{x}_t\|_{\bm{V}_t^{-1}}^2)\\
&=\log\frac{\det \bm{J}_T}{\det \bm{J}_0}
-\sum_{t=1}^T\log\frac{\det \bm{V}_t}{\det \bm{J}_{t-1}}.
\end{align*}

Let $\bm{Q}_U,\bm{Q}_\perp$ have orthonormal columns spanning
$U,U^\perp$, respectively. Define
$\bm{\mathsf B}_t=\bm{Q}_U^\top \bm{J}_t\bm{Q}_U$ and
$\bm{\mathsf F}_t=\bm{Q}_\perp^\top \bm{J}_t\bm{Q}_\perp$.
Schur complements give
\[
\det \bm{J}_t\le\det\bm{\mathsf B}_t\det\bm{\mathsf F}_t,
\qquad
\bm{Q}_U^\top \bm{J}_t^{-1}\bm{Q}_U\succeq\bm{\mathsf B}_t^{-1}.
\]
When $r=d$, we have \(
\det \bm{J}_t=\det\bm{\mathsf B}_t
\) and \(
\bm{Q}_U^\top \bm{J}_t^{-1}\bm{Q}_U=\bm{\mathsf B}_t^{-1}\).
For the ease of telescoping later, we set $z_t=2\lambda r+th+\frac{e^{-R_s\varepsilon}\sigma}{2}rb(t)$. Then
\begin{align*}
\tr\bm{\mathsf B}_t
&=2\lambda r+\sum_{s\le t}w_s\|\bm{Q}_U^\top \bm{x}_s\|_2^2+\frac{e^{-R_s\varepsilon}\sigma}{2}rb(t)
\le z_t,\\
\tr\bm{\mathsf F}_t&\le2\lambda (d-r)+th\qquad(r<d).
\end{align*}
Let $D_\perp=(d-r)\log(1+Th/(2\lambda (d-r)))$ for $r<d$,
and $D_\perp=0$ for $r=d$.
By \(\det \bm{A}\le\left(\frac{\operatorname{tr}\bm{A}}{m}\right)^m
\), we have
\[
\log\frac{\det \bm{J}_T}{\det \bm{J}_0}
\le D_\perp+r\log\frac{z_T}{z_0}.
\]

Since $\bm{V}_t=\bm{J}_{t-1}+\frac{e^{-R_s\varepsilon}\sigma}{2}q_t\bm{P}_U$, the determinant lemma gives
\begin{align*}
\log\frac{\det \bm{V}_t}{\det \bm{J}_{t-1}}
&=\log\det\left(\bm{I}_r+\frac{e^{-R_s\varepsilon}\sigma}{2}q_t\bm{Q}_U^\top \bm{J}_{t-1}^{-1}\bm{Q}_U\right)\\
&\ge\log\det\left(\bm{I}_r+\frac{e^{-R_s\varepsilon}\sigma}{2}q_t\bm{\mathsf B}_{t-1}^{-1}\right)\\
&\ge r\log\left(1+\frac{e^{-R_s\varepsilon}\sigma rq_t}{2\tr\bm{\mathsf B}_{t-1}}\right)
\ge r\log\left(1+\frac{e^{-R_s\varepsilon}\sigma rq_t}{2z_{t-1}}\right).
\end{align*}
The second inequality holds by the convexity of
$f(x)=\log(1+cq_t/x)$ on $(0,\infty)$.
Finally, $z_t=z_{t-1}+\frac{e^{-R_s\varepsilon}\sigma}{2}rq_t+h$ implies
\begin{align*}
\Phi_T
&\le D_\perp+r\sum_{t=1}^T
\log\frac{z_t}{z_{t-1}+\frac{e^{-R_s\varepsilon}\sigma}{2}rq_t}\\
&=D_\perp+r\sum_{t=1}^T
\log\left(1+\frac{h}
{2\lambda r+(t-1)h+\frac{e^{-R_s\varepsilon}\sigma}{2}rb(t)}\right)
\\&=D_T(h).
\end{align*}
Taking $w_t=\nu_{\mathrm{c}}(\bm{x}_t^\top\bar{\bm{\theta}}_t)$ proves the stated case.
\end{proof}

%% file: appendix/3-regret-analysis-lemma.tex
\begin{lemma}\label{lem:dot_mu_to_exp}
For $j\in\{\mathrm{c},\mathrm{d}\}$ and all $a,b$,
\[
e^{-R_s|a-b|}\dot\mu_j(b)\le\dot\mu_j(a)
\le e^{R_s|a-b|}\dot\mu_j(b).
\]
The same inequalities hold for $\nu_j=\dot\mu_j/a_j$.
\end{lemma}

\begin{proof}
Self-concordance gives $|(\log\dot\mu_j)'|\le R_s$.
Integrating between $a,b$ and exponentiating proves the claim.
\end{proof}

\begin{lemma}\label{lem:big_round_num}
On $\mathcal E$,
\[
\left|\left\{t_0<t\le T:
\|\bm{x}_t\|_{\bm{V}_t^{-1}}>\frac1{4G_T}\right\}\right|
\le\frac{16a_{\mathrm{c}}\kappa G_T^2}{\log2}D_T(1/(a_{\mathrm{c}}\kappa)).
\]
\end{lemma}

\begin{proof}
Define
\[
\tilde {\bm{V}}_t=2\lambda \bm{I}_d+
\frac1{a_{\mathrm{c}}\kappa}\sum_{s<t}\bm{x}_s\bm{x}_s^\top
+\frac{e^{-R_s\varepsilon}\sigma b(t)}2\bm{P}_U.
\]
Since $\bm{V}_t\succeq\tilde {\bm{V}}_t$, every round counted in the lemma satisfies
\[
\frac1{a_{\mathrm{c}}\kappa}\|\bm{x}_t\|_{\tilde {\bm{V}}_t^{-1}}^2
>\frac1{16a_{\mathrm{c}}\kappa G_T^2}.
\]
Theorem~\ref{lem:potential}, with $w_t=h=1/(a_{\mathrm{c}}\kappa)$, gives
\begin{align*}
&\left|\left\{t_0<t\le T:
\|\bm{x}_t\|_{\bm{V}_t^{-1}}>\frac1{4G_T}\right\}\right|
\log\left(1+\frac1{16a_{\mathrm{c}}\kappa G_T^2}\right)\\
&\quad\le\sum_{t=1}^T
\log\left(1+\frac1{a_{\mathrm{c}}\kappa}
\|\bm{x}_t\|_{\tilde {\bm{V}}_t^{-1}}^2\right)
\le D_T(1/(a_{\mathrm{c}}\kappa)).
\end{align*}
Since $a_{\mathrm{c}}\kappa G_T^2\ge\kappa R_{\dot\mu}\ge1$,
applying $\log(1+u)\ge u\log2$ for $u\in[0,1]$ proves the bound.
\end{proof}

%% file: appendix/4-additional-discussions.tex
\section{Additional Discussions}
\label{app:discussions}

\subsection{Potential Bound under Linear Query Schedules}
\label{app:linear_query}
Suppose $b(t)=\lfloor\rho t\rfloor$ with $\rho>0$ and $h>0$.
Write
\[
 z_t=2\lambda r+(t-1)(h+2e^{-R_s\varepsilon}\eta).
\]
Replacing $\lfloor\rho t\rfloor$ by $\rho t$ increases each
denominator in $D_T(h)^{(2)}$ by less than
$e^{-R_s\varepsilon}\sigma r/2$, so each new denominator is at
most $1+\frac{e^{-R_s\varepsilon}\sigma}{4\lambda}$ times its original value.
Using $\log(1+cx)\le c\log(1+x)$ for
$c\ge1$ and $x\ge0$, we obtain
\[
 D_T(h)^{(2)}
 \le\left(1+\frac{e^{-R_s\varepsilon}\sigma}{4\lambda}\right)
 r\sum_{t=1}^T\log\left(1+
 \frac{h}{z_t+2e^{-R_s\varepsilon}\eta}\right).
\]
By concavity of the logarithm,
\[
 \log\left(1+\frac{h}{z_t+2e^{-R_s\varepsilon}\eta}\right)
 \le\frac{h}{h+2e^{-R_s\varepsilon}\eta}
 \log\frac{z_{t+1}}{z_t}.
\]
Telescoping the sum yields
\[
 \sum_{t=1}^T\log\frac{z_{t+1}}{z_t}
 =\log\left(1+\frac{T(h+2e^{-R_s\varepsilon}\eta)}{2\lambda r}\right),
\]
which completes the proof.

\subsection{A Good Case}
\label{app:a_good_case}

We give a concrete setting in which pairwise conversational feedback
is much more informative than absolute arm reward.
Consider the logistic special case $\mu_{\mathrm{c}}(z)=\mu_{\mathrm{d}}(z)=1/(1+e^{-z})$, with $a_{\mathrm{c}}=a_{\mathrm{d}}=1$ and $R_s=1$. Let
\[
\bm{x}_a=\frac{1}{\sqrt{2}}(1,\bm{u}_a),
\qquad
\ts=\sqrt{2}(B,C\bm{v}),
\]
where $\bm{u}_a,\bm{v}\in\mathbb R^{d-1}$, $\|\bm{u}_a\|_2=\|\bm{v}\|_2=1$, and $B,C>0$.
Therefore, 
\begin{align*}
\|\ts\|_2&=\sqrt{2(B^2+C^2)},\qquad \|\bm{x}_a\|_2=1\\
\bm{x}_a^\top\ts&=B+C \bm{u}_a^\top \bm{v}\in[B-C,B+C].
\end{align*}
For $B\ge C$, all arm utilities are nonnegative. Since the logistic
curvature $\dot\mu_{\mathrm{c}}(z)=\mu_{\mathrm{c}}(z)(1-\mu_{\mathrm{c}}(z))$ decreases on
$[0,\infty)$, we have
\begin{equation*}\label{eq:app_arm_saturation}
\dot\mu_{\mathrm{c}}(\bm{x}_a^\top\ts)
\le \dot\mu_{\mathrm{c}}(B-C)\le e^{-(B-C)}.
\end{equation*}
Thus, the information in an absolute arm response can decay
exponentially with the common intercept $B$.

The definition of key-term vectors gives
\[
\tilde {\bm{x}}_k=\frac{1}{\sqrt{2}}(1,\tilde {\bm{u}}_k),
\qquad \|\tilde {\bm{u}}_k\|_2\le1.
\]
Suppose the relative key-term vectors $\tilde {\bm{u}}_k$ include
$\{\pm \bm{e}_j\}_{j=1}^{r}$, where $\bm{e}_j$ are the standard basis vectors
of $\mathbb R^{d-1}$. In a
pairwise comparison, the common intercept cancels out:
\begin{align*}
\bm{q}_{k,k'}
=\tilde {\bm{x}}_k-\tilde {\bm{x}}_{k'}
=\frac{1}{\sqrt{2}}(0,\tilde {\bm{u}}_k-\tilde {\bm{u}}_{k'}),\qquad
\bm{q}_{k,k'}^\top\ts
=C(\tilde {\bm{u}}_k-\tilde {\bm{u}}_{k'})^\top \bm{v}.
\end{align*}
Consequently,
\[
|\bm{q}_{k,k'}^\top\ts|\le2C,
\qquad
\dot\mu_{\mathrm{d}}(\bm{q}_{k,k'}^\top\ts)
\ge\dot\mu_{\mathrm{d}}(2C),
\]
and this lower bound is independent of $B$.

For fixed $C$ and a large $B$, we have $c_\star=\min_e\nu_{\mathrm{d}}(\bm{q}_e^\top\ts)\ge\dot\mu_{\mathrm{d}}(2C)=O(1)$, where $c_\star$ is defined in Appendix~\ref{app:algorithm_proofs}. 
Uniformly sampling the corresponding opposite pairs gives $\bm{V}(\bm{p})=2\bm{P}_U/r$, so $j_{\mathrm{geo}}\ge2/r$, where $j_{\mathrm{geo}}$ is defined in Appendix~\ref{app:algorithm_proofs}. 
Thus, for fixed $\xi>0$, Eq.~\eqref{eq:app_query_sigma} yields
\[
\sigma=\frac{\xi}{1+\xi}c_\star j_{\mathrm{geo}}\ge\frac{\xi}{1+\xi}\frac{2\dot\mu_{\mathrm{d}}(2C)}{r}
=\Omega(1/r).
\]

%% file: refs.bib
@inproceedings{xie2021comparison,
author = {Xie, Zhihui and Yu, Tong and Zhao, Canzhe and Li, Shuai},
title = {Comparison-based Conversational Recommender System with Relative Bandit Feedback},
year = {2021},
isbn = {9781450380379},
publisher = {Association for Computing Machinery},
address = {New York, NY, USA},
doi = {10.1145/3404835.3462920},
booktitle = {Proceedings of the 44th International ACM SIGIR Conference on Research and Development in Information Retrieval},
pages = {1400–1409},
numpages = {10},
location = {Virtual Event, Canada},
series = {SIGIR '21}
}

@inproceedings{lee2024a,
author = {Lee, Jungyhun and Yun, Se-Young and Jun, Kwang-Sung},
title = {A unified confidence sequence for generalized linear models, with applications to bandits},
year = {2024},
isbn = {9798331314385},
publisher = {Curran Associates Inc.},
address = {Red Hook, NY, USA},
booktitle = {Proceedings of the 38th International Conference on Neural Information Processing Systems},
articleno = {3960},
numpages = {46},
location = {Vancouver, BC, Canada},
series = {NIPS '24}
}

@inproceedings{yang2024conversational,
author = {Yang, Shuhua and Yuan, Hui and Zhang, Xiaoying and Wang, Mengdi and Zhang, Hong and Wang, Huazheng},
title = {Conversational Dueling Bandits in Generalized Linear Models},
year = {2024},
isbn = {9798400704901},
publisher = {Association for Computing Machinery},
address = {New York, NY, USA},
doi = {10.1145/3637528.3671892},
booktitle = {Proceedings of the 30th ACM SIGKDD Conference on Knowledge Discovery and Data Mining},
pages = {3806–3817},
numpages = {12},
location = {Barcelona, Spain},
series = {KDD '24}
}

@manual{tropp2011user,
  title        = {User-friendly Tail Bounds for Matrix Martingales},
  author       = {Tropp, Joel A.},
  month        = jan,
  year         = 2011,
  doi          = {10.7907/62v9-yh77},
}

@inproceedings{ liu2025leveraging,
author = {Liu, Maoli and Li, Zhuohua and Dai, Xiangxiang and Lui, John C.S.},
title = {Leveraging the Power of Conversations: Optimal Key Term Selection in Conversational Contextual Bandits},
year = {2025},
isbn = {9798400714542},
publisher = {Association for Computing Machinery},
address = {New York, NY, USA},
doi = {10.1145/3711896.3737025},
booktitle = {Proceedings of the 31st ACM SIGKDD Conference on Knowledge Discovery and Data Mining V.2},
pages = {1833–1844},
numpages = {12},
location = {Toronto ON, Canada},
series = {KDD '25}
}

@inproceedings{zhang2020conversational,
author = {Zhang, Xiaoying and Xie, Hong and Li, Hang and C.S. Lui, John},
title = {Conversational Contextual Bandit: Algorithm and Application},
year = {2020},
isbn = {9781450370233},
publisher = {Association for Computing Machinery},
address = {New York, NY, USA},
url = {https://doi.org/10.1145/3366423.3380148},
doi = {10.1145/3366423.3380148},
booktitle = {Proceedings of The Web Conference 2020},
pages = {662–672},
numpages = {11},
location = {Taipei, Taiwan},
series = {WWW '20}
}

@article{saha2021optimal,
  title={Optimal algorithms for stochastic contextual preference bandits},
  author={Saha, Aadirupa},
  journal={Advances in Neural Information Processing Systems},
  volume={34},
  pages={30050--30062},
  year={2021}
}

@book{mccullagh2019generalized,
  author    = {McCullagh, Peter and Nelder, John A.},
  title     = {Generalized Linear Models},
  edition   = {2},
  series    = {Chapman \& Hall/CRC Monographs on Statistics and Applied Probability},
  publisher = {Routledge},
  address   = {New York, NY},
  year      = {2019},
  pages     = {532},
  isbn      = {978-1-351-44584-9},
  doi       = {10.1201/9780203753736},
  url       = {https://doi.org/10.1201/9780203753736}
}

@misc{zeng2026best,
      title={Best Arm Identification in Generalized Linear Bandits via Hybrid Feedback}, 
      author={Qirun Zeng and Xuchuang Wang and Jiayi Shen and Xutong Liu and Fang Kong and Jinhang Zuo},
      year={2026},
      eprint={2605.05745},
      archivePrefix={arXiv},
}

@misc{feng2026llmroute,
      title={LLMRouter: Unified Infrastructure for Developing, Evaluating, and Deploying LLM Routers}, 
      author={Tao Feng and Fangxu Yu and Haozhen Zhang and Zhongjie Dai and Liangqi Yuan and Zijie Lei and Weizhi Zhang and Kunlun Zhu and Haodong Yue and Keyang Xuan and Ge Liu and Jiaxuan You},
      year={2026},
      eprint={2608.06867},
      archivePrefix={arXiv},
      primaryClass={cs.CL},
      url={https://arxiv.org/abs/2608.06867}, 
}

@misc{nguyen2025metallm,
      title={MetaLLM: A High-performant and Cost-efficient Dynamic Framework for Wrapping LLMs}, 
      author={Quang H. Nguyen and Thinh Dao and Duy C. Hoang and Juliette Decugis and Saurav Manchanda and Nitesh V. Chawla and Khoa D. Doan},
      year={2025},
      eprint={2407.10834},
      archivePrefix={arXiv},
}

@article{yue2012k,
  title={The k-armed dueling bandits problem},
  author={Yue, Yisong and Broder, Josef and Kleinberg, Robert and Joachims, Thorsten},
  journal={Journal of Computer and System Sciences},
  volume={78},
  number={5},
  pages={1538--1556},
  year={2012},
  publisher={Elsevier}
}

@inproceedings{wang2023efficient,
author = {Wang, Zhiyong and Liu, Xutong and Li, Shuai and Lui, John C.S.},
title = {Efficient explorative key-term selection strategies for conversational contextual bandits},
year = {2023},
isbn = {978-1-57735-880-0},
publisher = {AAAI Press},
doi = {10.1609/aaai.v37i8.26225},
booktitle = {Proceedings of the Thirty-Seventh AAAI Conference on Artificial Intelligence and Thirty-Fifth Conference on Innovative Applications of Artificial Intelligence and Thirteenth Symposium on Educational Advances in Artificial Intelligence},
articleno = {1156},
numpages = {8},
series = {AAAI'23/IAAI'23/EAAI'23}
}

@InProceedings{wang2025fusing,
  title = 	 {Fusing Reward and Dueling Feedback in Stochastic Bandits},
  author =       {Wang, Xuchuang and Zeng, Qirun and Zuo, Jinhang and Liu, Xutong and Hajiesmaili, Mohammad and Lui, John C.S. and Wierman, Adam},
  booktitle = 	 {Proceedings of the 42nd International Conference on Machine Learning},
  pages = 	 {65346--65367},
  year = 	 {2025},
  editor = 	 {Singh, Aarti and Fazel, Maryam and Hsu, Daniel and Lacoste-Julien, Simon and Berkenkamp, Felix and Maharaj, Tegan and Wagstaff, Kiri and Zhu, Jerry},
  volume = 	 {267},
  series = 	 {Proceedings of Machine Learning Research},
  month = 	 {13--19 Jul},
  publisher =    {PMLR},
  url = 	 {https://proceedings.mlr.press/v267/wang25eh.html},
}

@article{he2026learning,
  title={Learning across the gap: Hybrid multi-armed bandits with heterogeneous offline and online data},
  author={He, Qijia and Wang, Minghan and Liu, Xutong and Wang, Zhiyong and Kong, Fang},
  journal={Advances in Neural Information Processing Systems},
  volume={38},
  pages={20138--20182},
  year={2026}
}

@inproceedings{filippi2010parametric,
author = {Filippi, Sarah and Capp{\'e}, Olivier and Garivier, Aur{\'e}lien and Szepesv{\'a}ri, Csaba},
title = {Parametric bandits: the Generalized Linear case},
year = {2010},
publisher = {Curran Associates Inc.},
address = {Red Hook, NY, USA},
booktitle = {Proceedings of the 24th International Conference on Neural Information Processing Systems - Volume 1},
pages = {586–594},
numpages = {9},
location = {Vancouver, British Columbia, Canada},
series = {NIPS'10}
}

@InProceedings{russac2021self,
  title = 	 { Self-Concordant Analysis of Generalized Linear Bandits with Forgetting },
  author =       {Russac, Yoan and Faury, Louis and Capp{\'e}, Olivier and Garivier, Aur{\'e}lien},
  booktitle = 	 {Proceedings of The 24th International Conference on Artificial Intelligence and Statistics},
  pages = 	 {658--666},
  year = 	 {2021},
  editor = 	 {Banerjee, Arindam and Fukumizu, Kenji},
  volume = 	 {130},
  series = 	 {Proceedings of Machine Learning Research},
  month = 	 {13--15 Apr},
  publisher =    {PMLR},
}

@inproceedings{yue2011beat,
author = {Yue, Yisong and Joachims, Thorsten},
title = {Beat the mean bandit},
year = {2011},
isbn = {9781450306195},
publisher = {Omnipress},
address = {Madison, WI, USA},
booktitle = {Proceedings of the 28th International Conference on International Conference on Machine Learning},
pages = {241–248},
numpages = {8},
location = {Bellevue, Washington, USA},
series = {ICML'11}
}

@inproceedings{sui2018advancements,
author = {Sui, Yanan and Zoghi, Masrour and Hofmann, Katja and Yue, Yisong},
title = {Advancements in dueling bandits},
year = {2018},
isbn = {9780999241127},
publisher = {AAAI Press},
booktitle = {Proceedings of the 27th International Joint Conference on Artificial Intelligence},
pages = {5502–5510},
numpages = {9},
location = {Stockholm, Sweden},
series = {IJCAI'18}
}

@inproceedings{ong2024routellm,
 author = {Ong, Isaac and Almahairi, Amjad and Wu, Vincent and Chiang, Wei-Lin and Wu, Tianhao and Gonzalez, Joseph E and Kadous, M and Stoica, Ion},
 booktitle = {International Conference on Learning Representations},
 editor = {Y. Yue and A. Garg and N. Peng and F. Sha and R. Yu},
 pages = {34433--34448},
 title = {RouteLLM: Learning to Route LLMs from Preference Data},
 url = {https://proceedings.iclr.cc/paper_files/paper/2025/file/5503a7c69d48a2f86fc00b3dc09de686-Paper-Conference.pdf},
 volume = {2025},
 year = {2025}
}

@inproceedings{poon2026online,
author = {Poon, Manhin and Dai, Xiangxiang and Liu, Xutong and Kong, Fang and Lui, John C.S. and Zuo, Jinhang},
title = {Online multi-LLM selection via contextual bandits under unstructured context evolution},
year = {2026},
isbn = {978-1-57735-906-7},
publisher = {AAAI Press},
url = {https://doi.org/10.1609/aaai.v40i29.39672},
doi = {10.1609/aaai.v40i29.39672},
booktitle = {Proceedings of the Fortieth AAAI Conference on Artificial Intelligence and Thirty-Eighth Conference on Innovative Applications of Artificial Intelligence and Sixteenth Symposium on Educational Advances in Artificial Intelligence},
articleno = {2771},
numpages = {9},
series = {AAAI'26/IAAI'26/EAAI'26}
}
